\documentclass{article}

\usepackage{microtype}
\usepackage{graphicx}
\usepackage{subfigure}
\usepackage{booktabs} % for professional tables

\usepackage{hyperref}

\usepackage[accepted]{icml2024}

\usepackage{amsmath}
\usepackage{amssymb}
\usepackage{mathtools}
\usepackage{amsthm}

\usepackage[capitalize,noabbrev]{cleveref}

\theoremstyle{plain}
\newtheorem{theorem}{Theorem}[section]

\newtheorem{corollary}[theorem]{Corollary}
\theoremstyle{definition}

\theoremstyle{remark}

\usepackage[textsize=tiny]{todonotes}

\usepackage{amsmath,amsfonts,bm}

\def\eqref#1{equation~\ref{#1}}
\def\1{\bm{1}}

\def\rvk{{\mathbf{k}}}

\def\rvp{{\mathbf{p}}}
\def\rvq{{\mathbf{q}}}

\def\rvv{{\mathbf{v}}}
\def\rvw{{\mathbf{w}}}
\def\rvx{{\mathbf{x}}}

\def\vzero{{\bm{0}}}

\def\mI{{\bm{I}}}

\def\mPhi{{\bm{\Phi}}}

\DeclareMathAlphabet{\mathsfit}{\encodingdefault}{\sfdefault}{m}{sl}
\SetMathAlphabet{\mathsfit}{bold}{\encodingdefault}{\sfdefault}{bx}{n}

\def\gD{{\mathcal{D}}}

\def\gL{{\mathcal{L}}}

\def\gN{{\mathcal{N}}}
\def\gO{{\mathcal{O}}}

\def\gX{{\mathcal{X}}}

\def\sR{{\mathbb{R}}}

\newcommand{\E}{\mathbb{E}}

\DeclareMathOperator*{\argmin}{arg\,min}

\newcommand{\meanp}[2]{\mathbb{E}_{#1} \left\lbrack #2 \right\rbrack}

\newcommand{\kl}[2]{\mathrm{KL}\left(#1 || #2\right)}

\newcommand{\norm}[1]{\left\lVert#1\right\rVert}

\usepackage{bbm}
\newenvironment{claim}[1]{\par\noindent\underline{Claim:}\space#1}{}
\allowdisplaybreaks

\usepackage{subfigure}
\usepackage[T1]{fontenc}

\DeclareMathOperator{\subsample}{select}
\def\method{Gen-neG}
\def\oracle{\ensuremath{\gO}}

\def\synth{\ensuremath{{\widehat{\gD}}}}
\def\cls{\ensuremath{C}}

\icmltitlerunning{Score-based Generative Modeling with Oracle-assisted Guidance}

\begin{document}

\twocolumn[
\icmltitle{Don't be so Negative!\\Score-based Generative Modeling with Oracle-assisted Guidance}

\icmlsetsymbol{equal}{*}

\begin{icmlauthorlist}
\icmlauthor{Saeid Naderiparizi}{equal,ubccs,inverted}
\icmlauthor{Xiaoxuan Liang}{equal,ubccs,inverted}
\icmlauthor{Setareh Cohan}{ubccs}
\icmlauthor{Berend Zwartsenberg}{inverted}
\icmlauthor{Frank Wood}{ubccs,inverted,mila}
\end{icmlauthorlist}

\icmlaffiliation{ubccs}{Department of Computer Science, University of British Columbia, Vancouver, Canada}
\icmlaffiliation{inverted}{Inverted AI, Vancouver, Canada}
\icmlaffiliation{mila}{Montr\'eal Institute for Learning Algorithms (MILA)}

\icmlcorrespondingauthor{Saeid Naderiparizi}{saeidnp@cs.ubc.ca}

\icmlkeywords{Machine Learning, ICML}

\vskip 0.3in
]



\printAffiliationsAndNotice{\icmlEqualContribution} % otherwise use the standard text.

\begin{abstract}
Score-based diffusion models are a powerful class of generative models, widely utilized across diverse domains. Despite significant advancements in large-scale tasks such as text-to-image generation, their application to constrained domains has received considerably less attention.
This work addresses model learning in a setting where, in addition to the training dataset, there further exists side-information in the form of an oracle that can label samples as being outside the support of the true data generating distribution. Specifically we develop a new denoising diffusion probabilistic modeling methodology, \method{}, that leverages this additional side-information.
\method{} builds on classifier guidance in diffusion models to guide the generation process towards the positive support region indicated by the oracle. 
We empirically establish the utility of \method{} in applications including collision avoidance in self-driving simulators and  safety-guarded human motion generation.
\end{abstract}

\section{Introduction}\label{sec:intro}

What should we do when we train a generative model that generates samples known to be invalid within the constraints of the data domain?
For instance, when generating traffic scenes, road users cannot overlap each other. Likewise, in robotics, adherence to numerous physics-based constraints is essential for maintaining the appropriate motion and configuration of the robot. Typically, generative models are only trained to maximize the likelihood of a set of ``good'' training data samples.
Nevertheless, when sampling from a fully trained, highly expressive model, some fraction of generated samples fall into the category of ``bad'' samples.
Here we consider the problem of generative modeling where in addition to the conventional training dataset of good samples, we are also given access to constraints in the form of an oracle, which provides insights into whether a given sample is considered bad. Such oracles are ubiquitous in practice and are often a simple function implemented by domain experts.

Modern deep generative models are sufficiently parameterized that they can effectively be trained to result in a model placing a mixture of Dirac measures directly on the training data \citep{somepalli2023diffusion, somepalli2023understanding, carlini2023extracting}.
However, training such models on large amounts of data \citep{rombach2022high} or imposing regularization (such as smaller architectures or fewer integration steps) ensures that they generalize rather than memorize \citep{arpit2017closer,zhang2021understanding}. This also results in these models placing mass in the invalid part of the support \citep{hanneke2018actively}.
In this paper, we assume a modeling regime where the model generalizes effectively. Within this context, our objective is to reduce the probability mass assigned to invalid outputs while avoiding overfitting. Consequently, our contribution can be viewed as a method for controlling the specific type of the model's generalization.

The simplest way to use an oracle is to deploy the model with a rejection sampling loop in which the oracle is used to filter the output and return the constraint-satisfying samples~\citep{kim2023consistency}. 
Depending on circumstances this may constitute an acceptable final ``generative model'', but this solution comes at a (potentially unacceptable) computational cost.  Consider the concrete example of real-time autonomous vehicle path planning and model predictive control \citep{zhong2022guided}.
This task involves generating the next control action for an autonomous vehicle based the past and current state of itself and its surroundings. The generated action must avoid collisions and other types of invalid behavior, collectively referred to as ``infractions.'' This is an extremely challenging task that requires low latency and high success rates.
To guarantee low latency, it is essential to generate a sufficiently large number of parallel samples to ensure obtaining at least one valid sample with high probability.
Assume that a generative model trained on trajectories with no infractions produces infracting trajectories for all vehicles with probability $\epsilon$ (state of the art models \citep{lee2017desire, djuric2018short, gupta2018social, cui2019multimodal, ngiam2021scene, scibior2021imagining, niedoba2023diffusion} can have high infraction rates.)
Generating at least one non-infracting sample with $1-\delta$ probability without looping the rejection sampler requires $\frac{\log \delta}{\log \epsilon}$ parallel samples.  Depending on the specific concrete value of $1-\delta$ required (e.g. 1 chance in a billion of having latency arising from rejection sampling looping imposed) and the baseline trajectory model rejection rate (e.g. 30-50\% is not atypical) this could require running many parallel samplers (in this concrete example around 30). Depending on model size and available realtime edge computational capacity, this quantity may be prohibitively large.  Other examples of this nature arise in many control as inference problems \citep{levine2018reinforcement}.

Minimizing $\epsilon$ directly i.e., restricting the generative model to only place mass on the positive support region indicated by the oracle, is the most natural approach to combat this problem.  Working towards this goal includes a body of work on amortized rejection sampling \citep{warrington2020coping,naderiparizi2022amortized} and the body of related work on generative adversarial networks (GANs) \citep{goodfellow2014generative}.  Of course in the GAN setting, the discriminator (which can be used in a rejection sampling loop for improved performance \citep{azadi2018discriminator, che2020your}) is learned rather than being a fixed, pre-defined oracle as in the case we consider. 

\begin{figure*}
    \vspace*{-4mm}
    \centering
   \subfigure[]{
   \includegraphics[width=0.22\textwidth]{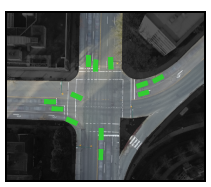}
   \label{fig:banner-a}
   }
   \subfigure[]{
   \includegraphics[width=0.22\textwidth]{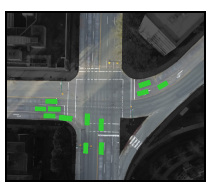}
   \label{fig:banner-b}
   }
   \subfigure[]{
   \includegraphics[width=0.22\textwidth]{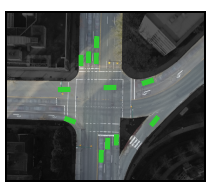}
   \label{fig:banner-c}
   }
   \subfigure[]{
   \includegraphics[width=0.22\textwidth]{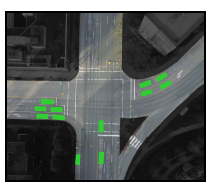}
   \label{fig:banner-d}
   } \\[-1em]
   \subfigure[]{
   \includegraphics[width=0.22\textwidth]{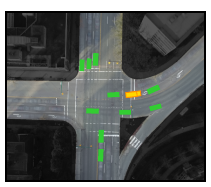}
   \label{fig:banner-e}
   }
   \subfigure[]{
   \includegraphics[width=0.22\textwidth]{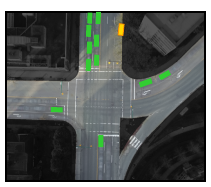}
   \label{fig:banner-f}
   }
   \subfigure[]{
   \includegraphics[width=0.22\textwidth]{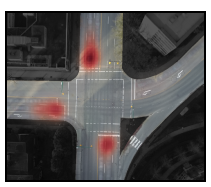}
   \label{fig:banner-g}
   }
   \subfigure[]{
   \includegraphics[width=0.22\textwidth]{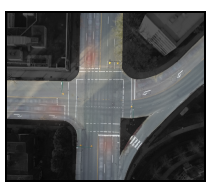}
   \label{fig:banner-h}
   }
   \vspace*{-4mm}
    \caption{\method{} applied to a diffusion model of non-infracting static vehicle placements (i.e.~a set of oriented rectangles) for the efficient initialization of autonomous vehicle planning simulators (see \citet{zwartsenberg2023conditional} for a similar model and full problem description). The top row show samples (green ``cars'') that are not colliding (non-overlapping) and not off-road (stay within the unshaded area of road surfaces) from a baseline diffusion model improved by \method{}.  The second row shows the kind of infractions our oracle identifies as not being in the support of the true distribution. \subref{fig:banner-e} shows a collision (yellow overlapping cars) and \subref{fig:banner-f} shows an off-road car in yellow.
    \subref{fig:banner-g} and \subref{fig:banner-h} graphically illustrate the reduction in infractions per unit area before and after  \method{} is applied to the baseline model (both plots are normalized to the same maximum value).  Quantitative results corresponding to this plot appear later in \cref{tab: 2nd experiment}.}
    \label{fig:banner}
\end{figure*}

We focus specifically on learning with constraints in score-based models.
What we reveal in this study is a \textit{necessary condition}, essential for the accurate functioning of classifier guidance in this problem domain, which, to the best of our knowledge, has been surprisingly overlooked until now.
Following recent findings on discriminator guidance in diffusion processes \citep{kim2022refining}, we introduce a new methodology for classifier guidance.
Our approach involves training and employing a series of differentiable classifiers, trained on synthetic samples generated from a sequence of classifier-guided diffusion models and labeled by the oracle.
The resulting sequence of multiply classifier-guided diffusion models effectively reduce the rejection rate while empirically maintaining a competitive probability mass assigned to validation samples. 
We demonstrate that comparable performance can be achieved with a reduced computational overhead by distilling the sequence of classifiers.

We evaluate our proposed methodology, which we call \textbf{Gen}erative modeling with \textbf{neG}ative examples (\method{})
on several problems, including modeling motion capture sequence data in a way that eliminates ground plane violations, and static traffic scene vehicle arrangements that avoid collisions and off-road placements.

\section{Background}

\subsection{Score-based Diffusion Models}\label{background:score-based}
Score-based diffusion models \citep{sohl2015deep, song2019generative, ho2020denoising, song2021scorebased}, also referred to as diffusion models (DMs) are a class of generative models that are defined through a stochastic process which gradually adds noise to samples from a data distribution $q_0(\rvx_0)$, such that when simulated forward from $t=0$ the marginal distribution at time $T$ is $q_T(\rvx_T) \approx \pi(\rvx_T)$ for some known $\pi(\rvx_T)$ typically equal to $\gN(\vzero, \mI)$. This is known as the ``forward process'' and is formulated as an SDE
\begin{equation}
    d\rvx_t = f(\rvx_t, t) dt + g(t) d \rvw, \quad \rvx_0 \sim q_0(\rvx_0),
    \label{eq:forward-process}
\end{equation}
where $f$ and $g$ are predefined drift and diffusion coefficients of $\rvx_t$ and $\rvw$ is the standard Wiener process.
DMs define another stochastic process known as the ``reverse process'' defined as
\begin{equation}
    d \rvx_t = [f(\rvx_t, t) - g(t)^2 s_\theta(\rvx_t; t)] dt + g(t) d\bar{\rvw}, \, \rvx_T \sim \pi(\rvx_T),
    \label{eq:reverse-process}
\end{equation}
where $\bar{\rvw}$ is the infinitesimal reverse time and reverse Wiener process, respectively. If $s_\theta$ matches the score function of the marginals of the forward process, the terminal distribution of the reverse process coincides with $q_0(\rvx_0)$~\citep{anderson1982reverse}.
Formally,
\begin{equation}
    s_\theta(\rvx_t; t) = \nabla_{\rvx_t} \log q_t(\rvx_t) \Rightarrow p_\theta(\rvx_0; 0) = q_0(\rvx_0),
\end{equation}
where $p_\theta(\rvx_t; t)$ is the marginal distribution of the reverse process.

In order to approximate the score function $\nabla_{\rvx_t} \log q_t(\rvx_t)$, DMs minimize the following score matching objective function~\citep{hyvarinen2005estimation,vincent2011connection,song2019generative}:
\begin{equation}
    \mathcal{L}^{\text{DM}}_\theta = \meanp{t,\rvx_0,\rvx_t} {\gamma_t \norm{s_\theta(\rvx_t; t) - \nabla_{\rvx_t} \log q(\rvx_t | \rvx_0)}^2},
    \label{eq:score-matching}
\end{equation}
where $\rvx_0 \sim q(\rvx_0)$, $\rvx_t \sim q(\rvx_t | \rvx_0)$, $t$ is sampled from a distribution over $[0, T]$, and $\gamma_t$ is a positive weighting term. Importantly, the Wiener process in \cref{eq:forward-process} allows direct sampling from the marginals of the forward distributions~\citep{song2021scorebased}, i.e. $q(\rvx_t | \rvx_0) = \gN(\alpha_t \rvx_0, \sigma_t)$, with $\alpha_t$ and $\sigma_t$ determined by the drift and diffusion coefficients in \cref{eq:forward-process}. This formulation moreover allows the evaluation of the conditional score function ($\nabla_{\rvx_t} \log q(\rvx_t | \rvx_0)$) in closed form.

Many of the DMs reported in the literature operate on discrete time steps~\citep{ho2020denoising,nichol2021improved}, and can be considered as particular discretizations of the presented framework.

In the remainder of this paper we use $q$ to denote the forward process, $s_\theta$ for the score function of the reverse process and $p_\theta$ as the distribution generated by running \cref{eq:reverse-process} backward in time. This applies to the marginals, conditionals, and posteriors as well. Furthermore, to reduce notational clutter throughout the rest of the paper, we omit the explicit mention of $\theta$ and $\phi$ and $t$ when their meaning is evident from the context.

\subsection{Classifier Guidance} \label{sec:background:classifier-guidance}

A distinctive and remarkable property of DMs is the ability to utilize an unconditional diffusion model to draw samples from its class-conditional distributions at inference time without requiring re-training or fine-tuning \citep{Dhariwal2021diffusion,song2021scorebased}. However, doing so typically utilizes a time-dependent classifier $q(y | \rvx_t) = \int q(y | \rvx_0) q(\rvx_0 | \rvx_t)\,d\rvx_0$
(alternative approaches include \citep{wu2023practical}).
Here, $q(y|\rvx_0)$ is a traditional classifier, that predicts the class probabilities for each $y$ given a datum $\rvx_0$ from the dataset. While $q(y | \rvx_t)$ classifies a noisy datum $\rvx_t$ sampled from $q_t(\rvx_t) = \int q(\rvx_t | \rvx_0) q(\rvx_0) \,d\rvx_0$.

Classifier guidance follows from the identity $\nabla_{\rvx_t}\log q(\rvx_t | y) = \nabla_{\rvx_t}\log q(\rvx_t) + \nabla_{\rvx_t}\log q(y | \rvx_t)$. Since the score function of the DM $s_\theta(\rvx_t; t) \approx \nabla_{\rvx_t}\log q(\rvx_t)$, we have
\begin{equation}
    s_\theta(\rvx_t | y; t) = s_\theta(\rvx_t; t) + \nabla_{\rvx_t}\log q(y | \rvx_t).
\end{equation}

\paragraph{Binary classification}
A special case of the above classifier guidance that we use in this paper is when there are only two classes. We provide here a brief overview of such a binary classification task and the notation associated with it.
Let $q(\rvx | y=1)$ and $q(\rvx | y=0)$ be the distribution of positive and negative examples, respectively. Let $\alpha = q(y = 1)$ and $1 - \alpha = q(y = 0)$ be the prior probabilities $q(y)$ of positive and negative examples. We then have $q(\rvx) = \alpha q(\rvx_t | y=1) + (1 - \alpha) q(\rvx_t | y=0)$. A binary classifier $C_\phi: \gX, [0, T] \rightarrow  [0, 1]$, can then be trained to approximate $q(y=1|\rvx_t)$ by minimizing the expected cross-entropy loss
\begin{equation}
\begin{split}
    \gL_{\phi}^{\text{CE}} =
    - \E_{t, \rvx_t} \Bigr[ q(y=1|\rvx_t) \log \cls_\phi(\rvx_t; t) + \\
    q(y=0|\rvx_t) \log (1 - \cls_\phi(\rvx_t; t)) \Bigr],
\end{split}
    \label{eq:cls-cross-entropy}
\end{equation}
where $\rvx_t \sim q(\rvx_t)$. Minimizing the cross-entropy loss between the classifier output and the true label is equivalent to minimizing the KL divergence between the classifier output and the Bayes optimal classifier \citep[Chapter~4]{sugiyama2012density}.
Therefore, \cref{eq:cls-cross-entropy} is minimized when
\begin{equation}
    \begin{split}
    \cls_{\phi^*}(\rvx_t; t) &= q(y = 1 | \rvx_t)\\
    &= \frac{\alpha q(\rvx_t | y=1)}{\alpha q(\rvx_t | y=1) + (1 - \alpha) q(\rvx_t | y=0)}.
    \label{eq:cls-optimal}
    \end{split}
\end{equation}
Hence, $s_\theta(\rvx_t | y=1; t) = s_\theta(\rvx_t; t) + \nabla_{\rvx_t} \log \cls_{\phi^*}(\rvx_t; t)$.
Note that this minimizer critically depends on the class prior probabilities $\alpha$ and $1 - \alpha$. \method{} works by ensuring that these are properly accounted for.

\section{Methodology}

In this section, we describe \textbf{Gen}erative modeling with \textbf{neG}ative examples (\method{}), a method for guiding the sampling of diffusion models to satisfy the constraints imposed by an oracle function.
\method{} has two stages. It starts by training a DM on available training data following standard DM training procedures (e.g. \cref{background:score-based}) without utilizing the oracle. We refer to this model as the ``baseline DM'' throughout.
In the second stage of \method{}, we draw samples from this baseline DM, label them using the oracle, and train a binary classifier using those samples, which we later use for guidance.
Next, we use the obtained classifier to guide our baseline DM, and the combination of both constitutes a new generative model.  \method{} establishes this combined model as a new DM and then repeats the process of sampling, classifying using the oracle, and training a time dependent classifier to form yet another model. 
Refinement using this iterative process can be repeated until the desired performance is obtained. We refer to this type of refinement as ``stacking''. Optionally, if better computational performance is desired, the stacked model can be distilled into a new model at any desired stage. 
As mentioned before, an important feature of \method{} is to properly account for the prior class probabilities $\alpha$ and $1 - \alpha$ in the training of all classifiers, which is formalized later in this section and demonstrated later in \cref{sec:experiments}. An overview of \method is shown in \cref{fig:gen-neg-overview} and \cref{alg: iterative training}.

\paragraph{Problem formulation and notation} Let $\gD = \{\rvx^i\}_{i=1}^N \sim q(\rvx)$ be a dataset of i.i.d. samples from an unknown data distribution $q$. Furthermore, let $\oracle: \gX \rightarrow \{0, 1\}$ be an oracle function that assigns each point in the data space $\gX$ a binary label. In other words, this oracle partitions the data space into two disjoint sets $\gX = \Omega \cup \Omega^\complement$ such that $\oracle(\rvx) = \1_{\Omega}(\rvx)$.
Moreover, assume $\gD \subseteq \Omega$ i.e., all training examples satisfy the oracle constraints.
Our objective is to learn a score-based diffusion model that (i) maximizes the likelihood of $\gD$ and (ii) avoids allocating probability to $\Omega^\complement$.

\subsection{Bayes Optimal Classifier Guidance for Diffusion models} \label{sec:method:classifier}
The core component of \method{} is in the second stage where a \emph{Bayes optimal} diffusion-time dependent classifier that discriminates between positive and negative samples in $\Omega$ and $\Omega^\complement$ respectively, is used to guide the baseline DM.
There are two main questions that \method{} answers. (i) Which data and class distribution should the classifier be Bayes optimal with respect to such that classifier guidance does not modify the sampling distribution in the oracle approved region? (ii) How can one train such a classifier?

Classifier guidance in score-based models is typically used to generate samples from a specific pre-defined class on the training dataset. As such, the classifier and generative model share the same training distribution.
In our case, however, there is no pre-defined dataset of positive and negative classes. Even the training dataset $\gD$ only includes samples from one class, since they all satisfy the oracle constraints.
Despite that, as we show later in \cref{sec:experiments}, classifier guidance using a Bayes optimal classifier for the binary classification task on the fully-synthetic data generated by the same baseline DM leads to (i) zero infraction and (ii) improved likelihood estimation on the data distribution, including a \textit{held-out test set}.
\begin{figure}[t]
    \centering
    \includegraphics[trim={27 10 35 15},clip,width=\linewidth]{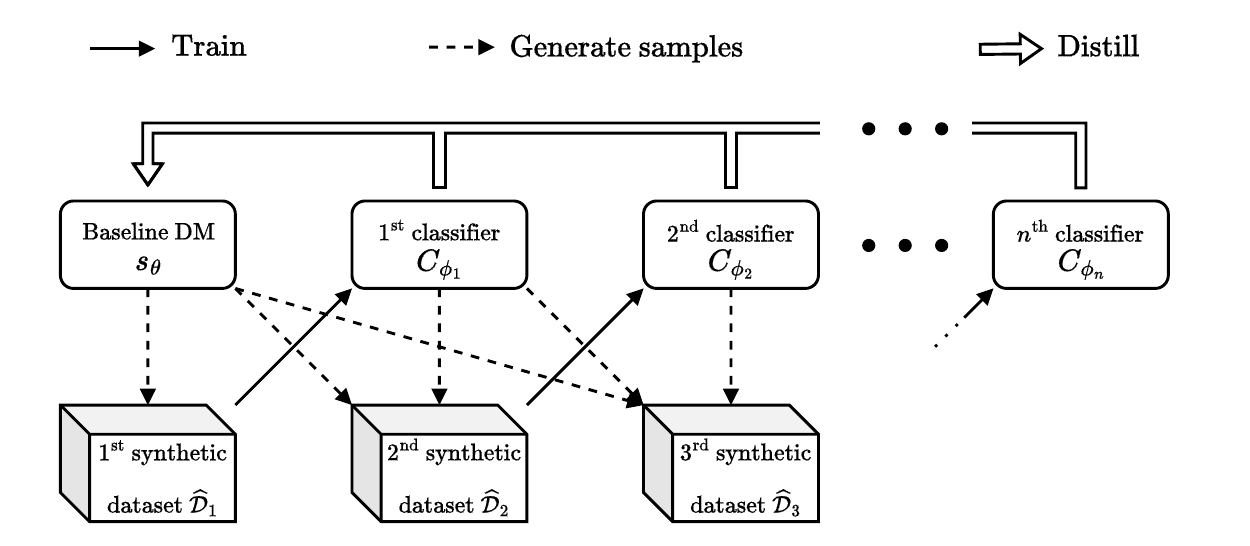}
    \caption{Overview of \method{}. The process begins with a baseline diffusion model. In each iteration, a synthetic training dataset is generated from the current model and labeled by the oracle function $\oracle$. A time-dependent classifier is trained on this dataset and then used to update the model by guiding it towards the positive support region (see \cref{eq:gen-neg-score}). The guided model (with multiple classifiers) can be distilled into a new, improved baseline diffusion model at any iteration by minimizing \cref{eq:distillation}.}
    \label{fig:gen-neg-overview}
\end{figure}

The goal of \method{} is therefore to solve this classification task on synthetic data. Formally, the data is distributed as $p_\theta(\rvx_0)$, the labels are $y = \oracle(\rvx_0)$, and the noise distribution is $p_\theta(\rvx_t | \rvx_0)$.
One can obtain the Bayes optimal classifier for this task using a cross-entropy objective similar to \cref{eq:cls-cross-entropy} in which $q$ is replaced by $p_\theta$. This objective can be equivalently written as
\begin{equation}
\begin{split}
    \gL_{\phi, \theta}^{\text{CE}} = - \E_{t, \rvx_0, \rvx_t} \Bigr[\oracle{}(\rvx_0) \log \cls_\phi(\rvx_t; t) + \\
    (1 - \oracle{}(\rvx_0)) \log (1 - \cls_\phi(\rvx_t; t)) \Bigr],
\end{split}
    \label{eq:cls-cross-entropy-synth}
\end{equation}
where $(\rvx_0, \rvx_t) \sim p_\theta(\rvx_0, \rvx_t)$ (see the Appendix for details). In order to avoid the computational cost of sampling from the baseline DM to compute this objective, \method{} approximates $p_\theta(\rvx_0, \rvx_t) \approx p_\theta(\rvx_0) q(\rvx_t | \rvx_0)$. This is similar to the approximation in \citet{kim2022refining}. The objective function of \method{} is therefore
\begin{equation}
\begin{split}
    \gL_\phi^{\text{cls}} = - \E_{t, p_\theta(\rvx_0)} \Bigr[ \E_{q(\rvx_t | \rvx_0)} \Big[ \oracle{}(\rvx_0) \log \cls_\phi(\rvx_t; t) + \\
    (1 - \oracle{}(\rvx_0)) \log (1 - \cls_\phi(\rvx_t; t)) \Big] \Bigr].
\end{split}
    \label{eq:gen-neg-true-objective}
\end{equation}
For notational simplicity, we drop the dependence of $\gL_\phi^{\text{cls}}$ on $\theta$, including in the equation above. Once trained, the classifier is incorporated into the baseline DM by
\begin{equation}
    s_{\theta, \phi}(\rvx_t; t) = s_\theta(\rvx_t; t) + \nabla_{\rvx_t} \log C_\phi(\rvx_t; t).
    \label{eq:gen-neg-score}
\end{equation}
We denote the marginal distributions generated by the oracle-assisted DM, implicitly defined through \cref{eq:gen-neg-score} as $p_{\theta,\phi}(\rvx_t; t)$.

\paragraph{Training the classifier}
Training the classifier in our approach presents a noteworthy challenge due to the major label imbalance within the synthetic dataset $\synth$ generated by the model. This imbalance emerges when the baseline is already close to the target distribution, resulting in a scarcity of negative examples. 
Meanwhile, these negative examples play a crucial role in guiding the model at the boundary between positive and negative examples, where the model requires the most guidance.

\method{} addresses this challenge by sampling a balanced dataset $\synth$ from the model, ensuring the same number of positive and negative examples.
However, this changes the class prior probabilities from the true marginal distribution over labels $\alpha$ and $1 - \alpha$ which in turn changes the optimal classifier the cross-entropy objective targets  (see \cref{eq:cls-optimal}).
We show evidence of this happening in \cref{fig:exp-checkerboard-samples}.
\method{} crucially employs importance sampling in the classifier's training objective to rectify the  bias introduced by having to balance the dataset to achieve high classifier accuracy in training.

\begingroup
\crefname{equation}{Eq.}{Eqs.}
\Crefname{equation}{Eq.}{Eqs.}

\begin{algorithm}[!t]
    \caption{\method{}}\label{alg: iterative training}
    \begin{algorithmic}[1]
        \STATE {\bfseries Input:} dataset $\gD$, oracle $\oracle$, synthetic dataset size $N$
        \STATE $i \leftarrow 0$
        \STATE $\theta_i \leftarrow \argmin_\theta \gL^{\text{DM}}_\theta$ \COMMENT{train baseline DM, \cref{eq:score-matching}}
        \STATE ${s_i} \leftarrow s_{\theta_i}(\rvx_t; t)$
        \REPEAT
            \STATE $\synth^+_i,\synth^-_i \leftarrow \varnothing$
            \REPEAT
              \STATE $\synth^+,\synth^- \leftarrow$ generate more samples  from DM with score function ${s_i}$ and label with $\oracle$
              \STATE $\synth^+_i \leftarrow \synth^+_i \cup \synth^+$, $\synth^-_i \leftarrow \synth^-_i \cup \synth^-$
            \UNTIL{$\min(|\synth^+_i|,|\synth^-_i|)<N$}
            \STATE $\alpha_i \leftarrow |\synth^+_i| / (|\synth^+_i| + |\synth^-_i|)$ \COMMENT{Estimate class prior probabilities for Bayes optimal classifier} 
            \STATE $\synth^+_i \leftarrow \subsample(N,\synth^+_i), \synth^-_i \leftarrow \subsample(N,\synth^-_i)$ \COMMENT{balance dataset for IS classifier training}
            \STATE $\phi_i \leftarrow \argmin_\phi \hat{\gL}^{\text{cls}}_\phi(\alpha_i, \synth^+_i, \synth^-_i)$ \COMMENT{train guidance classifier, \cref{eq: classifier loss}}
            \STATE $i \leftarrow i+1$
            \IF{distill}
            \STATE $\psi \leftarrow \argmin_\psi \gL^{\text{dtl}}_\psi$ \COMMENT{See \cref{eq:distillation}}
            \STATE ${s_i} \leftarrow  s_{\psi}(\rvx_t; t)$ 
            \ELSE
                % \COMMENT{``stack'' guidance classifiers}
                \STATE ${s_i} \leftarrow {s_{i-1}} + \nabla_{\rvx_t}\log C_{\phi_i}(\rvx_t; t)$ \COMMENT{See \cref{eq:gen-neg-score}}
            \ENDIF
        \UNTIL{done}
        \STATE {\bfseries Output:} DM score function $s_i$
    \end{algorithmic}
\end{algorithm}

\endgroup % End of local abbreviation settings

Given a balanced dataset $\synth = \synth^+ \cup \synth^-$ where $\synth^+ \sim p(\rvx_0 | y=1)$, $\synth^- \sim p(\rvx_0 | y=0)$, $N = |\synth^+| = |\synth^-|$, and $\alpha = p_\theta(y=1)$,
\begin{align}
    \label{eq: classifier loss}
    &\hat{\gL}_{\phi}^{\text{cls}}(\alpha, \synth^+, \synth^-) \nonumber\\
    :=& \frac{1}{N} \sum_{\rvx_0 \in \synth^+} \alpha \meanp{q(\rvx_t|\rvx_0)}{- \log \cls_\phi(\rvx_t; t)}\nonumber\\
    + &\frac{1}{N} \sum_{\rvx_0 \in \synth^-} (1 - \alpha) \meanp{q(\rvx_t|\rvx_0)}{- \log (1 - \cls_\phi(\rvx_t; t))},
\end{align}
is an importance sampling estimator of the objective function in \cref{eq:gen-neg-true-objective}; proof in \cref{app:objective-is}.

\subsection{Iterative Training by Stacking Classifiers} With an optimal classifier minimizing \cref{eq:cls-cross-entropy-synth}, the DM with score function $s_{\theta,\phi}$ will have improved likelihood and zero infraction (see \cref{thm:method:improved-model}). However, in practice the trained classifier is only an estimate because learning the true decision boundary would require an infeasible synthetic dataset size and optimization budget, thus infractions may not be entirely eliminated.

To alleviate this problem, we note that once the classifier is trained, the guided score function $s_{\theta,\phi}(\rvx)$ itself defines a new diffusion model.
Consequently, we can employ a similar procedure to train a new classifier on $s_{\theta,\phi}$, aiming to further lower its infraction rate. This iterative approach involves training successive classifiers and incorporating them into the model, progressively enhancing its performance and reducing the infraction rate.

\begin{figure*}[t]
    \vspace*{-5mm}
    \centering
    \subfigure[]{
    \includegraphics[width=0.2\textwidth]{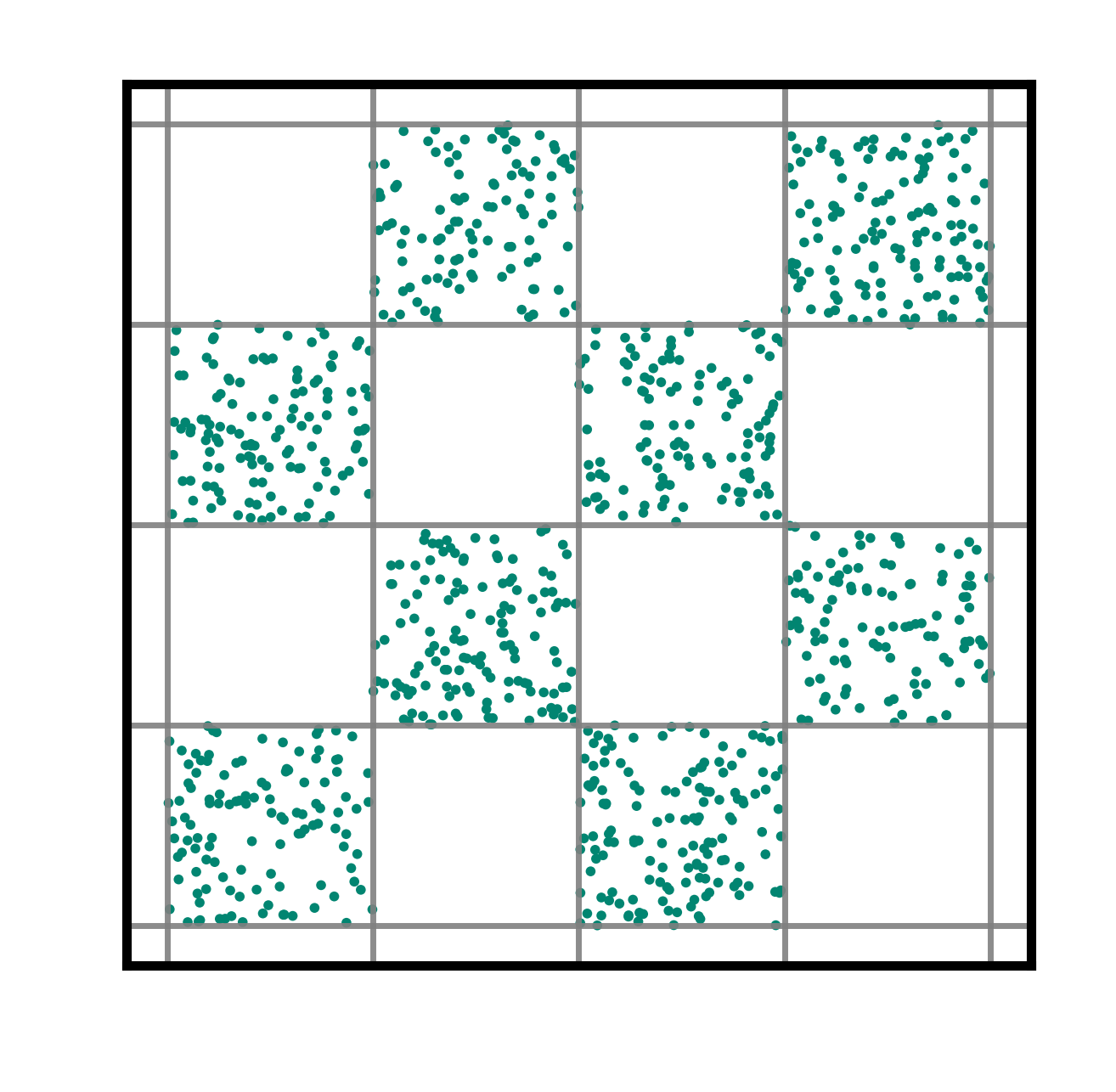}
    \vspace*{-4mm}
    \label{fig:exp-checkerboard-samples:a}}
    \hfill
    \subfigure[]{
    \includegraphics[width=0.2\textwidth]{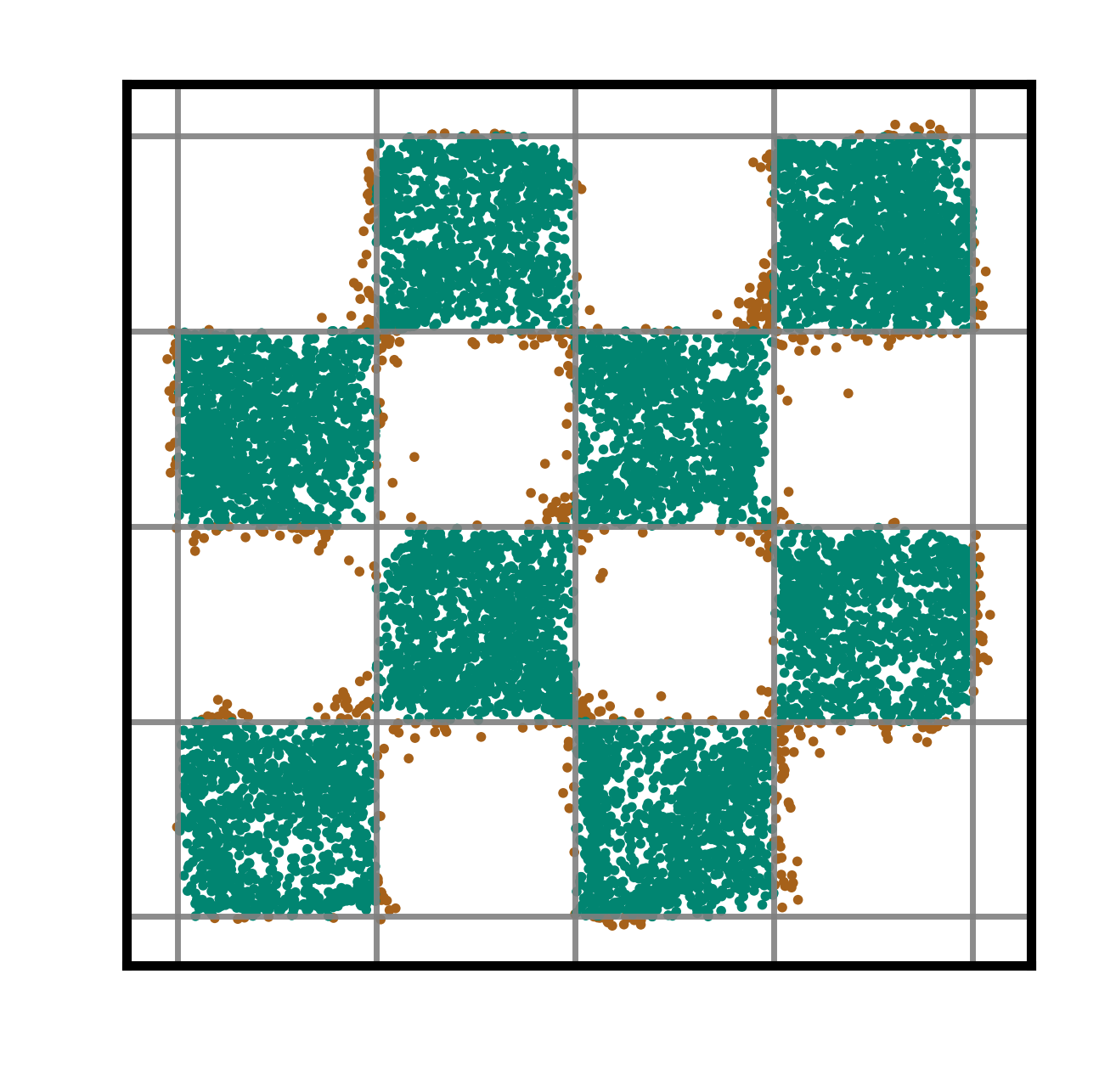}
    \vspace*{-4mm}
    \label{fig:exp-checkerboard-samples:b}
    }
    \hfill
    \subfigure[]{
    \includegraphics[width=0.2\textwidth]{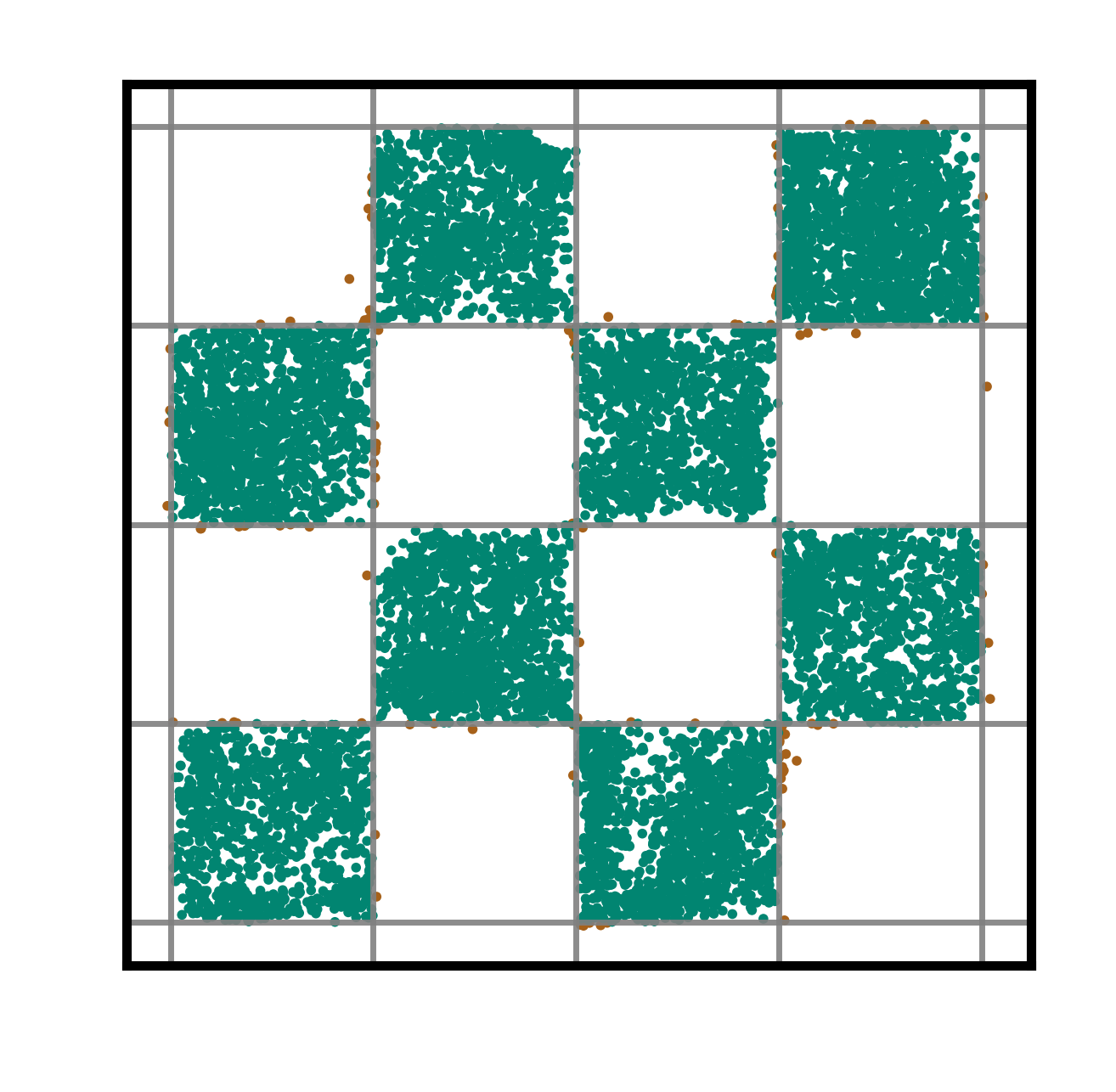}
    \vspace*{-4mm}
    \label{fig:exp-checkerboard-samples:c}
    }
    \hfill
    \subfigure[]{
    \includegraphics[width=0.2\textwidth]{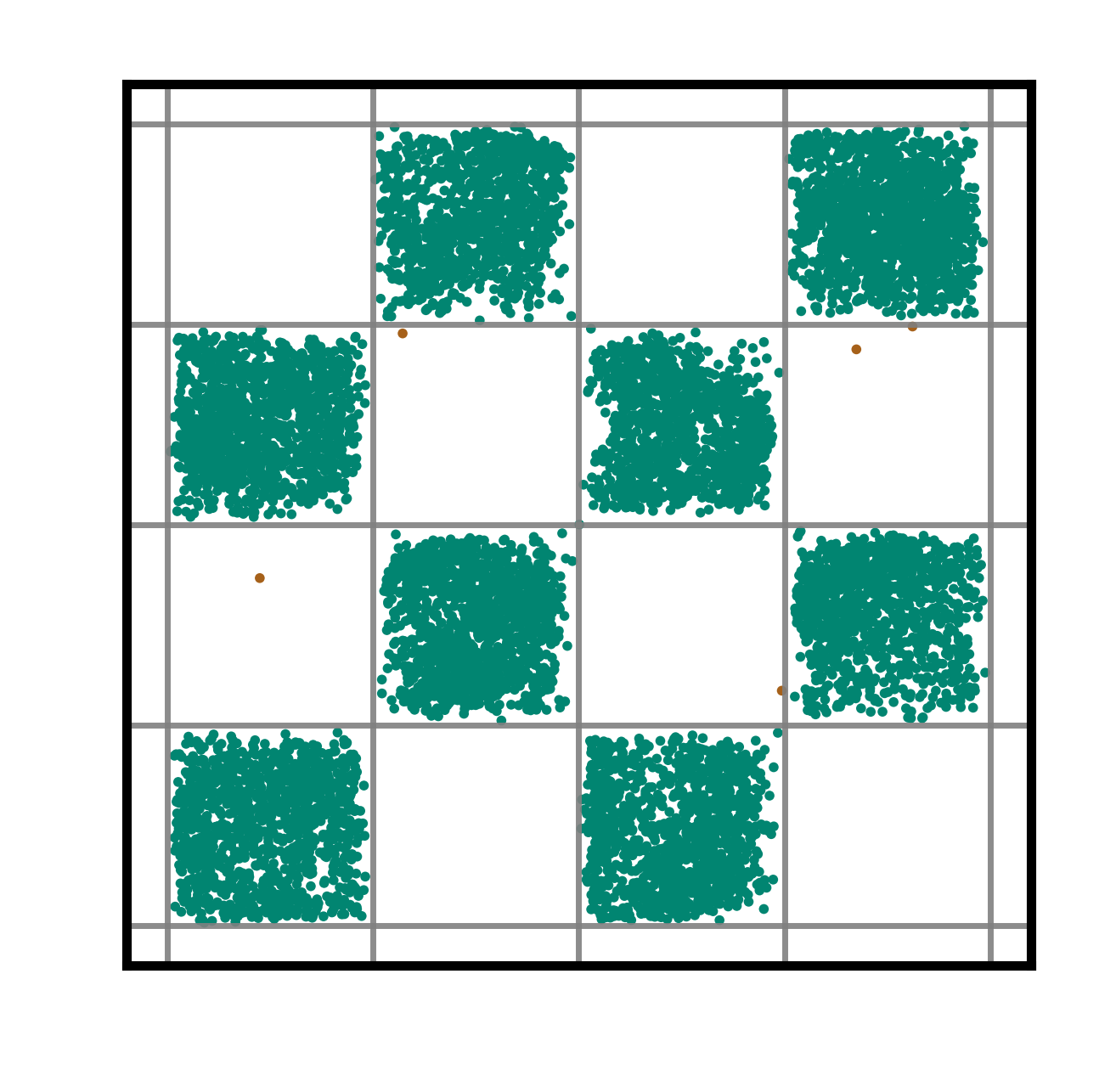}
    \vspace*{-4mm}
    \label{fig:exp-checkerboard-samples:d}}
    \vspace*{-4mm}
    \caption{Samples from the checkerboard experiment. Samples with infraction (i.e. $\oracle(\rvx) = 0$) are shown in brown.~\subref{fig:exp-checkerboard-samples:a} The baseline training dataset; \subref{fig:exp-checkerboard-samples:b} baseline DM; \subref{fig:exp-checkerboard-samples:c} first iteration of \method{} using a Bayes optimal classifier trained on a balanced dataset and correct $\alpha$; \subref{fig:exp-checkerboard-samples:d} a classifier trained on a balanced dataset without employing importance sampling results in suboptimal density estimation. We see samples are suboptimally pushed inwards from the boundaries.  We also have observed that validation ELBOs in these kinds of cases are significantly worse. }
    \vspace*{-4mm}
    \label{fig:exp-checkerboard-samples}
\end{figure*}

\subsection{Model Distillation}
Adding a stack of classifiers to the model linearly increases its computational cost, since each new classifier requires a forward and backward pass each time the score function is evaluated. To avoid this, we show that it is possible and sometimes beneficial to distill the classifiers into a combined diffusion model.

Let $s_{\theta, \mPhi}$ be a ``teacher model'' consisting of a baseline model $s_{\theta}$ and a stack of classifiers $\{\cls_\phi\}_{\phi \in \mPhi}$. We distill $s_{\theta, \mPhi}$ into a new ``student model'' $s_{\psi}^{\text{dtl}}$, possibly with the same architecture as the baseline model, by minimizing the following distillation loss
\begin{equation}
    \mathcal{L}_\psi^{\text{dtl}} = \meanp{\mathbf{x}_0\sim q(\mathbf{x}_0), t} {\gamma_t \norm{s_{\theta, \mPhi}(\mathbf{x}_t; t) - s_\psi^{\text{dtl}}(\mathbf{x}_t; t)}^2},
    \label{eq:distillation}
\end{equation}
where $\gamma_t$ is the weight term, similar to the training objective of diffusion models. Here, $\mathcal{L}^{\text{dtl}}$ makes the outputs of the student model match that of the teacher.
Algorithm~\ref{alg: iterative training} summarises \method{}.

\subsection{Theory}
Here we provide the theoretical grounding for why the classifier \method{} targets results in improved likelihood estimation and avoids violating the constraints.
\begin{theorem}
    \label{thm:method:cls-guidance}
    Let $p_\theta(\rvx)$ be the distribution learned by a baseline DM with marginal distributions denoted by $p_\theta(\rvx_t; t)$ and let $p_\theta(y=1 | \rvx_0)= \oracle(\rvx_0)$. Further, let $\cls_{\phi^*}: \gX, [0, T] \rightarrow [0, 1]$ be the Bayes-optimal time-dependent binary classifier arising from perfectly optimizing the following cross-entropy objective
    \begin{equation}
    \begin{split}
        \gL_{\phi, \theta}^{\text{CE}}= -\E_t \Big[ \E_{p_\theta(\rvx_0, \rvx_t)} \big[ \oracle(\rvx_0) \log \cls_\phi(\rvx_t; t) + \\
        (1 - \oracle(\rvx_0)) \log (1 - \cls_\phi(\rvx_t; t)) \big] \Big]
        \label{eq:cls-time-cross-entropy}
    \end{split}
    \end{equation}
    then
    \begin{equation}
    \begin{split}
        \nabla_{\rvx_t} \log p_\theta(\rvx_t | y=1 ; t) = \nabla_{\rvx_t} \log p_\theta(\rvx_t; t) +  \\
        \nabla_{\rvx_t} \log \cls_{\phi^*}(\rvx_t; t).
    \end{split}
    \end{equation}
\end{theorem}
In other words, by using a Bayes-optimal binary classifier for guidance, we target exactly the score function of positive (oracle-approved) examples, without modifying the underlying distribution in the oracle-approved region.
\begin{corollary}
    \label{thm:method:improved-model}
    For an optimal classifier $\cls_{\phi^*}$, 
    \begin{enumerate}
        \item $p_{\theta, \phi^*}(\rvx_t) = p_{\theta}(\rvx_t | y=1)$,
        \item There is no mass on $\Omega^\complement$, i.e.  $\int_{\rvx\in\Omega^\complement} p_{\theta, \phi^*} (\rvx) = 0$,
        \item For any dataset $\gD \subseteq \Omega$, $p_{\theta, \phi^*}(\gD) \geq p_{\theta}(\gD)$.
    \end{enumerate}
\end{corollary}
\cref{thm:method:improved-model} suggests that our \method{} methodology can improve the baseline DM in terms of both infraction rate and test dataset likelihood.

See the proofs for the \cref{thm:method:cls-guidance,thm:method:improved-model} in \cref{app:proof-thm,app:proof-corr}.

\section{Experiments}
\label{sec:experiments}
We evaluate \method{} on three datasets: a 2D checkerboard, collision avoidance in traffic scenario generation, and safety-guarded human motion generation.
In each experiment we report a likelihood-based metric on a held out dataset to measure distributional shifts and a form of infraction metric which measures faithfulness to the oracle. We release the source code for the checkerboard and motion generation experiments\footnote{\url{https://github.com/plai-group/gen-neg}}.

\subsection{Checkerboard}\label{sec:exp:toy}

\begin{figure*}[t]
    \centering
    \includegraphics[trim={0 165 0 0},clip,scale=0.72]{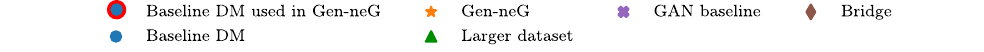}\\
    \includegraphics[trim={0 0 240 22},clip,scale=0.75]{figs/toy-infraction-both-captioned.pdf}\hfill
    \includegraphics[trim={0 0 444 22},clip,scale=0.75]{figs/toy-infraction-both-captioned.pdf}
    \includegraphics[trim={274 0 0 22},clip,scale=0.75]{figs/toy-infraction-both-captioned.pdf}
    \caption{Infraction, ELBO and MMD estimates from the checkerboard experiment.
    Dashed lines connect different iterations of the same method. For the baseline DM, it corresponds to training iterations. For \method{} it corresponds to different iterations of our algorithm.
    The plot shows (i) prolonged training of the \textit{baseline DM} reduces the infraction rate but leads to overfitting.
    (ii) Our experiments on \textit{larger datasets}, denoted by triangles in varying shades of green (up to 1000$\times$ larger than the original dataset; darker shades indicate larger datasets), maximum likelihood training even on substantially larger datasets is strongly outperformed by \method{}.
    (iii) Various iterations of \method{} consistently decrease the infraction rate while maintaining fidelity to the data distribution.
    (iv) Diffusion bridges perfectly achieve zero infraction rate and improved likelihood estimation. However, they require analytical access to constraints and are not generalizable to complex constraints.
    }
    \label{fig:exp-checkerboard-scores}
\end{figure*}

To develop some insight, we start by demonstrating the principles and performance of \method{} on a dataset of 2-dimensional points uniformly distributed on a checkerboard grid as shown in \cref{fig:exp-checkerboard-samples:a}.
We apply EDM \citep{karras2022elucidating}, a continuous-time DM, to this problem.
The training dataset only contains 1000 points.
This makes the model prone to over-fitting. As shown in \cref{fig:exp-checkerboard-scores} and further explored in \cref{app:overfitting}(blue dots), training the model for long causes strong overfitting.
We therefore stop training of the baseline DM before it starts overfitting measured by the evidence lower bound (ELBO) on a held-out validation set. \cref{fig:exp-checkerboard-samples:b} shows samples from this baseline DM and the blue dot with a red edge in \cref{fig:exp-checkerboard-scores} shows its measured performance.

\cref{fig:exp-checkerboard-samples:c} shows samples from the first iteration of \method{} and the orange stars in \cref{fig:exp-checkerboard-scores} show the metrics for the first five iterations of \method{}. These results show \method{} dramatically reduces the rate of infractions while still matching the data distribution for non-infracting regions.
\cref{fig:2d-ablation-imbalance} in the appendix shows distilling various \method{} models maintains a comparable performance.

We test the effectiveness \method{} against training on a larger dataset. The green triangles in \cref{fig:exp-checkerboard-scores} show the performance of models trained on datasets up to $10^6$ points ($1000\times$ larger than our original dataset). Even the first iteration of \method{} achieves significantly lower infraction rates compared to any of these models.
This also emphasizes the importance of negative samples, consistent with \citet{giannone2023learning}.

An alternative approach to learning constrained distributions with diffusion models is diffusion bridges \citep{liu2023learning} that provably produce no infraction. However, it requires analytical access to the constraints and quantities that are only tractable under very simple constraints. As such it is not applicable to our problem setting. Therefore, it should be treated as an upper bound to \method{}'s performance. As shown by the brown diamond in \cref{fig:exp-checkerboard-scores}, this method achieves zero infractions with a better ELBO.

\citet{kong2023data} proposed an algorithm for data redaction in GANs that is applicable to our oracle-based constraints. Since GANs do not provide ELBO estimates, we only compute MMD and infraction rates for this baseline. We train this model on our problem and choose the two checkpoints with the best value for either metric. The purple crosses in the right panel of \cref{fig:exp-checkerboard-scores} show our results. The MMD scores are significantly worse than those of diffusion-based models, including \method{}. Additionally, \method{} quickly outperforms this baseline in infraction rate.

The \method{} models in this experiment reach near-zero infraction rates. This makes the balanced synthetic dataset generation step slow. We explore an importance sampling-based approach to avoid this slowdown. Our approach and its results are presented in \cref{fig:app:faster-sampling} but we leave further exploration for future research.

\begin{table*}[!t]
  \caption{Results for traffic scene generation, in terms of collision, offroad, and overall infractions as well as reweighted ELBO (r-ELBO). We compare \method{} against a normalizing flow baseline~\citep{zwartsenberg2023conditional}, a classifier trained on imbalanced synthetic dataset, and a classifier without importance sampling and a time-independent classifier. The final two rows provide the results of distilling the models labeled with $\dagger$ and *.}
  \centering
  \small
  \label{tab: 2nd experiment}
  \begin{tabular}{lllll}
    \toprule
    Method     & Collision ($\%$) $\downarrow$  & Offroad ($\%$) $\downarrow$ & Infraction ($\%$) $\downarrow$ & r-ELBO ($\times 10^{-2}$) $\uparrow$  \\
    \midrule
    Baseline DM & $28.3\pm 0.70$  & $1.3\pm 0.14$ & $29.3\pm 0.64$ &  $-27.5\pm0.01$  \\
    Normalizing flow~\citep{zwartsenberg2023conditional} & $91.2\pm 0.27$ & $13.1\pm 0.48$ & $91.9\pm 0.25$ & --- \\
    \midrule
    Time-independent classifier & $20.7\pm 0.59$& $0.9\pm 0.09$ & $21.4\pm 0.63$ & $-244\pm 30.4$ \\
    Imbalanced classifier (ablation) & $17.8 \pm 1.21$ & $0.9\pm 0.16$ & $18.6\pm 1.30$ & $-27.7\pm 0.01$        \\
    w/o IS classifier (ablation) & $14.6\pm 0.49$ & $0.8\pm 0.13$ & $15.2\pm 0.50$ & $-28.0\pm 0.01$  \\
    \method{}$^\dagger$ (iteration 1)    & $16.4\pm 0.5$ &  $0.9\pm 0.12$ & $17.2\pm 0.44$ & $-27.7\pm 0.01$  \\
    \method{}$^*$ (iteration 2)  & $11.6\pm 0.65$ & $0.6\pm 0.10$ & $12.2\pm 0.60$ & $-28.0\pm 0.01$ \\
        \midrule
    \method{}\ (distillation of $\dagger$) &$12.2\pm0.42 $ & $0.8\pm0.06$ &$12.9\pm0.36$ & $\mathbf{-26.8\pm 0.01}$  \\
    \method{}\ (distillation of *) & $\mathbf{5.1\pm 0.24}$ & $\mathbf{0.5\pm 0.09}$ & $\mathbf{5.6\pm 0.20}$ & $-27.0\pm 0.01$\\
    \bottomrule
  \end{tabular}
\end{table*}

\subsection{Traffic Scene Generation}
\label{sec:IC-experiment}
We continue to the task of traffic scene generation where vehicles of varying sizes are placed on a given two-dimensional map.
Traditionally implemented by rule based systems~\citep{yang1996microscopic, lopez2018microscopic}, this task has recently been approached using generative modeling techniques~\citep{tan2021scenegen, zwartsenberg2023conditional}. 
In both of these prior works, the common approach has been to discard any samples that violate predefined rules, such as a vehicle being outside the designated driving area (``offroad'') or overlapping with another vehicle (``collision'').
Rejecting such samples, while effective, can be computationally wasteful, particularly when rule violations occur frequently. Hence, in this context, we employ \method{} to enhance performance.
The specific task we consider is to generate $12$ vehicles in a given scene, conditioned on a rendered representation of the roadway.
Each vehicle is represented by its position, length, width, orientation and velocity for a total of $7$ dimensions per vehicle.
Vehicles are sampled \emph{jointly}, meaning that the features are in $\mathbb{R}^{N\times 7}$. 
We train the baseline DM employing the formalism in DDPM~\citep{ho2020denoising} with a transformer-based denoising  network~\citep{vaswani2017attention} on a private dataset. Our architecture consists of self-attention layers and map-conditional cross-attention layers in an alternating order.
We use relative positional encodings (RPEs)~\citep{shaw2018self, wu2021rethinking}.
Further details are provided in \cref{app:sec:details-ic}.
Relevant examples (including infracting, and non-infracting ones) and road geometry can be seen in \cref{fig:banner}.

\cref{tab: 2nd experiment} summarizes the results of this experiment. We compare against the baseline DM model and various guided models. Further, we compare against a prior work on this problem based on normalizing flows (NFs) \citep{zwartsenberg2023conditional}. Comparing to the baseline DM, \method{} lowers the infraction rates while maintaining a comparable distribution match. \method{} significantly outperforms the NF baseline because DMs are much more expressive generative models. The infraction rates reported here for the NF baseline are worse than those of \citet{zwartsenberg2023conditional}. This can be attributed to the higher average traffic density in our dataset compared to the INTERACTION dataset~\citep{zhan2019interaction} which \citet{zwartsenberg2023conditional} uses. In \cref{sec:app:ic-interaction} we empirically verify this by training the baseline DM on the INTERACTION dataset.

In the second section of \cref{tab: 2nd experiment} we report results of various guided diffusion models using a synthetic dataset generated by the baseline DM and labelled by the oracle. First, we consider a common approach of classifier guidance in which a time-independent classifier pre-trained \textit{on clean data} is utilized to perform (approximate) classifier guidance \citep{wu2023practical,bansal2023universal}.
In this approach, one-step estimate of $\rvx_0 \approx \frac{\rvx_t + \sigma_t^2 s_\theta(\rvx_t; t)}{\alpha_t}$ is obtained using the diffusion model. This estimate is then passed to the pre-trained classifier.
This method enhances the infraction rate but exhibits a significant decrease in ELBO, indicating a strong distributional misalignment. We also present ablations where we omit the importance sampling (``w/o IS'') step or forego balancing the dataset (``imbalanced classifier'') in \method{}. The ``w/o IS'' ablation improves infraction rates, but they both deteriorate the ELBO.

Different iterations of \method{}, however, shows even better infraction rates. The presence of lower ELBO can be justified by the approximations in \method{}'s objective function and classifiers not being trained to optimality. This is why the results deviate from theory to some extent. On the other hand, training the baseline DM is the only stage where we explicitly maximize the ELBO. Classifiers trained on all the other iterations only implicitly improve ELBO through guiding the model to not allocate probability mass on the invalid region. Finally, in the third section of \cref{tab: 2nd experiment} we demonstrate that our approach of distilling the resulting models back into a single one works well here too, sometimes even surpassing their teacher models. This can be attributed to knowledge distillation effects \citep{hinton2015distilling}.
Overall we find that \method{} works as expected, and provides a competitive infraction rate boost over our baseline model, without shifting the distribution.
To relate this to the introduction, as explained in \cref{sec:app:computational-cost}, using \method{} in producing non-infracting scenes for autonomous vehicle synthetic data generation would reduce GPU costs by 57\% on average.

\subsection{Motion Diffusion}
Our final experiment focuses on human motion generation. Diffusion models have been successfully applied to motion generation and editing tasks \citep{tevet2023human, zhang2024motiondiffuse, shafir2023human, xie2023omnicontrol, cohan2024flexible}. While these models produce diverse and realistic results, they often lack physical plausibility \citep{yuan2022physdiff}. For instance, issues like ground penetration frequently occur in the generated examples. Such imperfections can affect the quality of the generated motions and limit the model's applicability in real-world scenarios.

\begin{table*}[!t]
  \caption{Results of the Motion Diffusion experiment. ``Inf. per step'' is the average rate of generated motion frames with infraction while ``infraction'' is the average rate of generated motions that at least including one infracting frame. r-ELBO is a reweighted ELBO with the same weighting as in diffusion loss.}
  \centering
  \small
  \label{tab:mdm}
  \begin{tabular}{llllll}
    \toprule
    Method     & Infraction ($\%$) $\downarrow$ & Inf. per step ($\%$) $\downarrow$ & r-ELBO ($\times 10^{-2}$) $\uparrow$  & FID $\downarrow$  & KID ($\times 10^{-3}$) $\downarrow$\\
    \midrule
    MDM (baseline DM) & $27.66 \pm 0.77$  & $7.84 \pm 0.27$ & $-1.06 \pm 0.02$ & $0.445 \pm 0.040$ & $8.27 \pm 2.14$\\
    \method{} (Ours) & $24.25 \pm 0.35$ & $6.12 \pm 0.19$ & $\mathbf{-1.01 \pm 0.03}$ & $\mathbf{0.414 \pm 0.030}$ & $\mathbf{6.99 \pm 0.78}$ \\
    w/o IS (ablation)  & $\mathbf{22.85 \pm 0.18}$ & $\mathbf{5.47 \pm 0.18}$ & $-1.13 \pm 0.06$ & $\mathbf{0.415 \pm 0.030}$ & $8.40 \pm 1.83$\\
    \bottomrule
  \end{tabular}
\end{table*}

For our baseline DM, we use the pre-trained checkpoints provided by Motion Diffusion Model (MDM) \citep{tevet2023human}, tailored for text-conditioned motion generation. MDM is a DDPM model with a transformer-based architecture trained on the HumanML3D dataset \citep{guo2022generating}. It uses a pre-trained CLIP embedding module \citep{radford2021learning} for conditioning on the text descriptions.
To address the issue of ground penetration, we implement an oracle that labels motions with ground penetration at any point in their duration as negative.
We employ \method{} with a classifier having the same architecture as MDM, but the CLIP encoder, as the classifier is not text-conditional.

\cref{tab:mdm} summarizes our results of one iteration of \method{} on this dataset.
We report infraction rate, reweighted ELBO (referring to a uniform schedule of $\gamma_t$ in \cref{eq:score-matching}). We also report Fr\'echet Inception Distance (FID) \citep{heusel2017gans}, and Kernel Inception Distance (KID) \citep{binkowski2018demystifying} to measure quality of samples. \method{} improves on all metrics. While the ablation of omitting the IS weighting produces lower infraction rates compared to \method{}, it worsens the reweighted ELBO and KID. Hence, \method{} improves infraction, with a improved model likelihood and sample quality.
We conjecture the relatively smaller improvement in the motion experiment is because the baseline DM predicts $\rvx_0$ \citep{zhong2022guided}.

\section{Related Work}
\citet{hanneke2018actively} proposes a theoretical framework for oracle-based constraints. They however, do not provide practical considerations.
More recently, use of negative and invalid data have been explored to improve the training of generative models. \citet{sinha2021negative} uses heuristic functions to augment the training set of GANs with negative samples, while \citep{giannone2023learning} utilizes a pre-defined negative set.
Meanwhile, data redaction approaches propose methods for removing undesirable learned concepts from pre-trained generative models in safety and security applications \citep{gandikota2023erasing,schramowski2023safe}. Similarly, \citep{kong2023data} explores various data redaction methods, with the validity-based approach being the most relevant to our oracle-assisted guidance, although in the context of GAN literature. They implicitly carry out data redaction by integrating it into the discriminator and fine-tuning the generator. Another approach to constrained generative modeling explicitly incorporates the constraints in the model, similar to a final layer that maps to the constraint set \citep{stoian2024realistic}.

Several studies have explored constrained score-based modeling employing techniques such as diffusion bridges \citep{wu2022diffusion,liu2023learning}, barrier methods \citep{fishman2023diffusion}, or reflected diffusion \citep{lou2023reflected,fishman2023diffusion} or mirror diffusion \citep{liu2023mirror}.
Despite being effective, they rely on constraint-specific information such as closed form, linear, or convex constraints. This imposes strong limitations, making them impractical for general problems where such information is unavailable.

\section{Conclusion}
We have proposed a framework to incorporate constraints into diffusion models. These constraints are defined through an oracle function that categorizes samples as either \emph{good} or \emph{bad}. Importantly, such a flexibility allows for simple integration with human feedback. We have demonstrated our model on different modalities demonstrating how it can benefit safety constraints.

The current limitations we recognize, and the possible future directions for this work are (i) incorporating the true training dataset into the later iterations of the method, as the training dataset only affects the baseline DM. The next stages solely use synthetic data. Although we show theoretically that our guidance only improves the model, this lack of revisiting the true dataset in presence of practical errors and approximations poses challenges for large-scale adoption of our method. Our preliminary experiments of visiting the true dataset at the distillation time have not been successful yet. (ii) Avoiding stacking of classifiers, instead directly learning an artifact that can replace the previous classifier in our method, similar to \citep{de2021diffusion}, is vital to the computational  complexity of the method as the current computational cost scales linearly with the number of classifiers. (iii) Extending \method{} on tabular diffusion to support tabular data~\citep{kotelnikov2209tabddpm}, which compasses a mixture of continuous and categorical data to generalize our work is an inspiring area to future research. (iv) Bridging the gap the diffusion bridge-based approaches and our work which is practically applicable to a larger set of applications is another avenue for future developments.

\section*{Acknowledgments}
We acknowledge the support of the Natural Sciences and Engineering Research Council of Canada (NSERC), the Canada CIFAR AI Chairs Program, Inverted AI, MITACS, the Department of Energy through Lawrence Berkeley National Laboratory, and Google. This research was enabled in part by technical support and computational resources provided by the Digital Research Alliance of Canada Compute Canada (alliancecan.ca), the Advanced Research Computing at the University of British Columbia (arc.ubc.ca), and Amazon.

\section*{Impact Statement}
Our work is intended to {\em improve} generative modeling performance by eliciting improved generalization performance.  Better generalization performance will lead to energy savings as less computation is required to generate ``good'' samples and better performance of systems that can be deployed for societal good such as self-driving cars that crash less often, robotic exoskeletons that are more safely and accurately performant.

\bibliography{bib}
\bibliographystyle{icml2024}

\newpage
\onecolumn
\appendix
\section{Proofs}
\subsection{Proof of Theorem~\ref{thm:method:cls-guidance}}\label{app:proof-thm}
\begin{proof}
    Let $\alpha$ and $(1-\alpha)$ be the prior probabilities of positive and negative examples under $p_\theta(\rvx_t; t)$. Note that $\alpha$ remains independent of $t$ because
    \begin{align*}
        \alpha & = \int p_\theta(\rvx_t; t) p_\theta(\rvx_0 | \rvx_t) p_\theta(y=1 | \rvx_0)\, d\rvx_0 \,d\rvx_t
        = \int p_\theta(\rvx_t; t) p_\theta(\rvx_0 | \rvx_t) \oracle(\rvx_0)\, d\rvx_0 \,d\rvx_t                 \\
               & = \int p_\theta(\rvx_0, \rvx_t; t) \oracle(\rvx_0) \,d\rvx_0 d\rvx_t
        = \int p_\theta(\rvx_0; 0) \oracle(\rvx_0) \,d\rvx_0.
    \end{align*}
    The objective in \cref{eq:cls-time-cross-entropy} is equal to
    \begin{multline}
        -\meanp{t}{\meanp{p_\theta(\rvx_t)}{\meanp{p_\theta(\rvx_0 | \rvx_t)}{\oracle(\rvx_0)} \log \cls_\phi(\rvx_t; t) + \meanp{p_\theta(\rvx_0 | \rvx_t)}{(1 - \oracle(\rvx_0))} \log (1 - \cls_\phi(\rvx_t; t))}}\\
        = -\meanp{t}{\meanp{p_\theta(\rvx_t)}{p(y=1|\rvx_t) \log \cls_\phi(\rvx_t; t) + p(y=0|\rvx_t) \log (1 - \cls_\phi(\rvx_t; t))}}
    \end{multline}
    This is equivalent to \cref{eq:cls-cross-entropy} after replacing $q$ with $p_\theta$, i.e. sampling from the reverse process instead of the forward. Therefore, its optimal solution follows \cref{eq:cls-optimal}. Hence,
    \begin{align*}
            & \ \nabla_{\rvx_t} \log p_\theta(\rvx_t; t) + \nabla_{\rvx_t} \log \cls_{\phi^*}(\rvx_t; t)                                                              \\
        =\  & \nabla_{\rvx_t} \log p_\theta(\rvx_t; t) + \nabla_{\rvx_t} \log \alpha p_\theta(\rvx_t | y=1; t)                                                          \\
            & \hspace{3cm} - \nabla_{\rvx_t} \log \Big[ \overbrace{\alpha p_\theta(\rvx_t|y=1; t) + (1 - \alpha) p_\theta(\rvx_t|y=0; t)}^{p_\theta(\rvx_t; t)} \Big] \\
        =\  & \nabla_{\rvx_t} \log \alpha p_\theta(\rvx_t | y=1; t) = \nabla_{\rvx_t} \log p_\theta(\rvx_t | y=1; t)
    \end{align*}
\end{proof}

\subsection{Proof of Corollary~\ref{thm:method:improved-model}}\label{app:proof-corr}

\begin{proof}
    Since $p_{\theta, \phi^*}$ is defined as the distribution generated by simulating the SDE in \eqref{eq:reverse-process}, its score function $\nabla_{\rvx_t} \log p_{\theta, \phi^*}(\rvx_t; t)$ is by definition equal to $s_{\theta, \phi^*}(\rvx_t; t)$ \citep{risken1996fokker, song2019generative}. Similarly for the baseline DM we have $s_{\theta}(\rvx_t; t) = \nabla_{\rvx_t} \log p_\theta(\rvx_t; t)$.
    Therefore,
    \begin{align}
        \nabla_{\rvx_t} \log p_{\theta, \phi^*}(\rvx_t; t)
         & = s_{\theta, \phi^*}(\rvx_t; t)
        \overset{\text{\cref{eq:gen-neg-score}}}{=\joinrel=} s_\theta(\rvx_t; t) + \nabla_{\rvx_t} \log C_{\phi^*}(\rvx_t; t) \\
         & = \nabla_{\rvx_t} \log p_{\theta}(\rvx_t; t) + \nabla_{\rvx_t} \log C_{\phi^*}(\rvx_t; t)
        \overset{\text{Thm. \ref{thm:method:cls-guidance}}}{=\joinrel=} \nabla_{\rvx_t} \log p_{\theta}(\rvx_t | y=1)
    \end{align}
    Here we derived $s_{\theta, \phi^*}(\rvx_t; t) = \nabla_{\mathbf{x}_t} \log p_\theta(\rvx_t | y=1)$. By~\citep{anderson1982reverse}, we proved the first statement.

    The second statement  follows by decomposing $p_{\theta}(\rvx_t | y=1)$:
    \begin{align*}
        p_{\theta, \phi^*}(\rvx) = p_\theta(\rvx | y=1) \propto p_\theta(\rvx) \oracle(\rvx)\quad
        \Rightarrow \quad p_{\theta,\phi^*}(\rvx) = 0\quad \forall \rvx \in \Omega^\complement.
    \end{align*}
    For the last statement, we have
    \begin{align*}
        & \left.\begin{matrix}
        p_{\theta, \phi^*}(\rvx) = \frac{p_{\theta}(\rvx) \oracle(\rvx)}{\int p_{\theta}(\rvx) \oracle(\rvx) \,d \rvx} \\[1.1em]
        \int p_{\theta}(\rvx) \oracle(\rvx) \,d \rvx \leq 1 \\
                \end{matrix}\right\}
        \Rightarrow p_{\theta, \phi^*}(\rvx) \geq p_{\theta}(\rvx)\quad \forall \rvx \in \Omega  \\
        \overset{\gD \subseteq \Omega}{\Longrightarrow} & \log p_{\theta, \phi^*}(\gD)
        = \sum_{\rvx \in \gD} \log p_{\theta, \phi^*}(\rvx)
        \geq \sum_{\rvx \in \gD} \log p_{\theta}(\rvx) = \log p_{\theta}(\gD).
    \end{align*}
\end{proof}

\begin{figure}[t]
    \centering
    \includegraphics{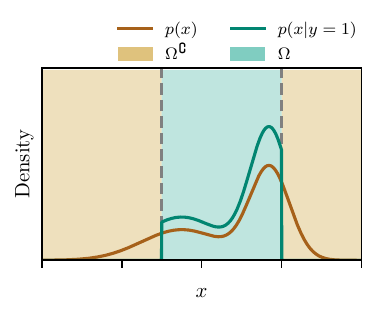}
    \caption{We use this one-dimensional density plot to show how we guide the generation process towards the positive support region indicated by the oracle. The base distribution $p(x)$ is a mixture of two Gaussian distributions shown by the brown curve. We show the regions allowed and disallowed by the oracle respectively by the cyan  and light brown shaded areas. Gen-neG with a Bayes optimal classifier targets the distribution $p(x | y = 1)$ that has probability mass in the allowed region and assigns zero probability outside this region.}
    \label{fig:one-dimensional}
\end{figure}

We demonstrate this corollary with a one-dimension density in \cref{fig:one-dimensional}. We show a ``base distribution'' $p(x)$ and the positive and negative regions $\Omega$ and $\Omega^\complement$, respectively. We can see that the distribution $p(x|y=1)$ assigns no mass to $\Omega^\complement$ and has a larger mass assigned to any point in $\Omega$.

\subsection{Equivalence of cross-entropy loss minimization and KL divergence minimization}\label{app:bayes-optimal-kl}
This is a well-known result in the literature. Nonetheless, we include the result and its proof here for completeness and ease of reference.
\begin{claim}
    minimizing the cross-entropy loss between the classifier output and the true label is equivalent to minimizing the KL divergence between the classifier output and the Bayes optimal classifier.
\end{claim}

\begin{proof}
    The Bayes optimal classifier $C^*(\rvx_t; t)$ approximates $q(y = 1 | \rvx_t)$. Let $p_\phi(y | \rvx_t)$ be the distribution represented by the learned classifier $C_\phi(\rvx_t; t)$ i.e., $p_\phi(y=1|\rvx_t) = C_\phi(\rvx_t; t)$. For an arbitrary diffusion time step $t$, the expected KL divergence between the Bayes optimal and the learned classifier therefore is
    \begin{align}
          & \meanp{q(\rvx_t)}{\kl{q(y | \rvx_t)}{p_\phi(y | \rvx_t)}}\nonumber  \\
        = & \ \meanp{q(\rvx_t)}{\meanp{q(y | \rvx_t)}{\log\frac{q(y|\rvx_t)}{p_\phi(y | \rvx_t)}}} \\
        = & \ \meanp{q(\rvx_t)}{q(y=1|\rvx_t)\log\frac{q(y=1|\rvx_t)}{C_\phi(\rvx_t; t)} + q(y=0|\rvx_t) \log\frac{q(y=0|\rvx_t)}{1 - C_\phi(\rvx_t; t)}}  \\
        = & \ \meanp{q(\rvx_t)}{H(q(y | \rvx_t))} - \meanp{q(\rvx_t)}{q(y=1 | \rvx_t) \log C_\phi(\rvx_t; t) + q(y=0 | \rvx_t)\log (1 - C_\phi(\rvx_t; t))}.
    \end{align}
    The first term is the expected entropy of the optimal classifier and is independent of $\phi$. Therefore,
    \begin{align}
          & \argmin_\phi \meanp{q(\rvx_t)}{\kl{q(y | \rvx_t)}{p_\phi(y | \rvx_t)}}\nonumber   \\
        = & \ \argmin_\phi - \meanp{q(\rvx_t)}{q(y=1 | \rvx_t) \log C_\phi(\rvx_t; t) + q(y=0 | \rvx_t)\log (1 - C_\phi(\rvx_t; t))} \\
        = & \ \argmin_\phi \text{CE}(q(y|\rvx_t), p_\phi(y|\rvx_t)).
    \end{align}
\end{proof}
Note that \cref{eq:cls-cross-entropy} is the expected cross entropy for different time steps $t$.

\subsection{Connection between Eq.~(\ref{eq: classifier loss}) and Eq.~(\ref{eq:cls-time-cross-entropy})}
In this section we make the connection between \cref{eq:gen-neg-true-objective} and \cref{eq:cls-time-cross-entropy} more clear. We start with \cref{eq:cls-time-cross-entropy}, ignoring the outer expectation with respect to $t$, is equal to

\begin{align}
    & - \Big(\meanp{p_\theta(\rvx_0, \rvx_t)}{\gO(\rvx_0)\log \cls_\phi(\rvx_t;t)+ (1 - \gO(\rvx_0))\log(1 - \cls_\phi(\rvx_t; t)}\Big)  \\
    = & \ -\textcolor{blue}{\int p_\theta(\rvx_0, \rvx_t)} \Big( \textcolor{blue}{p_\theta(y=1|\rvx_0)}\log\cls_\phi(\rvx_t; t) + \textcolor{blue}{p_\theta(y=0 | \rvx_0)}\log(1 - \cls_\phi(\rvx_t;t)) \Big)\,d\rvx_0d\rvx_t \\
    = & \ -\int \textcolor{blue}{p_\theta(y=1) p_\theta(x_0 | y=1) p_\theta(x_t | x_0)} \log\cls_\phi(\rvx; t)\,d\rvx_0d\rvx_t \nonumber \\
    & \ -\int \textcolor{blue}{p_\theta(y=0) p_\theta(x_0 | y=0) p_\theta(x_t | x_0)} \log(1 - \cls_\phi(\rvx_t;t))\,d\rvx_0d\rvx_t  \\
    \approx & \ -\int p_\theta(y=1) p_\theta(x_0 | y=1) \textcolor{blue}{q(x_t | x_0)} \log\cls_\phi(\rvx; t)\,d\rvx_0d\rvx_t \nonumber \\
    & \ -\int p_\theta(y=0) p_\theta(x_0 | y=0) \textcolor{blue}{q(x_t | x_0)} \log(1 - \cls_\phi(\rvx_t;t))\,d\rvx_0d\rvx_t  \\
    = & \  \alpha \meanp{p_\theta(\rvx_0|y=1)}{\meanp{q(\rvx_t|\rvx_0)}{- \log \cls_\phi(\rvx_t; t)}}  + (1 -\alpha) \meanp{p_\theta(\rvx_0 | y=0)}{\meanp{q(\rvx_t|\rvx_0)}{- \log (1 - \cls_\phi(\rvx_t; t))}} \\
    :=& \frac{1}{N} \sum_{\rvx_0 \in \synth^+} \alpha \meanp{q(\rvx_t|\rvx_0)}{- \log \cls_\phi(\rvx_t; t)} + \frac{1}{N} \sum_{\rvx_0 \in \synth^-} (1 - \alpha) \meanp{q(\rvx_t|\rvx_0)}{- \log (1 - \cls_\phi(\rvx_t; t))},
\end{align}
which recovers \cref{eq: classifier loss}.

\subsection{\method{}'s objective function}\label{app:objective-is}
Here we show why \cref{eq: classifier loss} is an importance sampling estimator of the original objective function in \cref{eq:gen-neg-true-objective}.

\begin{align}
    \gL_{\phi}^{\text{cls}}(\alpha) := & \ \alpha \meanp{p_\theta(\rvx_0|y=1)}{\meanp{q(\rvx_t|\rvx_0)}{- \log \cls_\phi(\rvx_t; t)}}\nonumber  \\
    +  & \ (1 - \alpha) \meanp{p_\theta(\rvx_0 | y=0)}{\meanp{q(\rvx_t|\rvx_0)}{- \log (1 - \cls_\phi(\rvx_t; t))}} \\
    =  & \ - \meanp{p_\theta(y)}{ \meanp{p_\theta(\rvx_0|y)}{\meanp{q(\rvx_t | \rvx_0)}{y \log C_\phi(\rvx_t; t) + (1-y) \log (1 - C_\phi(\rvx_t; t)}} }.
\end{align}
Now we apply importance sampling to $p_\theta(y)$ by sampling from $\pi(y)$ as the proposal distribution. Therefore,
\begin{align}
    \gL_{\phi}^{\text{cls}}(\alpha) = & \ - \meanp{p_\theta(y)}{ \meanp{p_\theta(\rvx_0|y)}{\meanp{q(\rvx_t | \rvx_0)}{y \log C_\phi(\rvx_t; t) + (1-y) \log (1 - C_\phi(\rvx_t; t)}} }  \\
    = & \ - \meanp{\pi(y)}{ \frac{p_\theta(y)}{\pi(y)} \meanp{p_\theta(\rvx_0|y)}{\meanp{q(\rvx_t | \rvx_0)}{y \log C_\phi(\rvx_t; t) + (1-y) \log (1 - C_\phi(\rvx_t; t)}} } \\
    = & \ \frac{p_\theta(y=1)}{\pi(y=1)} \meanp{p_\theta(\rvx_0|y=1)}{\meanp{q(\rvx_t | \rvx_0)}{-\log C_\phi(\rvx_t; t)}}\nonumber \\
    & + \frac{p_\theta(y=0)}{\pi(y=0)} \meanp{p_\theta(\rvx_0|y=0)}{\meanp{q(\rvx_t | \rvx_0)}{-\log (1 - C_\phi(\rvx_t; t))}} \\
    = & \ \meanp{p_\theta(\rvx_0|y=1)}{ \frac{\alpha}{\pi(y=1)} \meanp{q(\rvx_t | \rvx_0)}{-\log C_\phi(\rvx_t; t)}}\nonumber \\
    & + \meanp{p_\theta(\rvx_0|y=0)}{ \frac{1 - \alpha}{\pi(y=0)} \meanp{q(\rvx_t | \rvx_0)}{-\log (1 - C_\phi(\rvx_t; t))}} \label{eq:loss-is-proof}
\end{align}
In our case, $\pi(y)$ is a uniform Bernoulli distribution i.e., $\pi(y=1) = \pi(y=0) = 0.5$. Therefore, minimizing \cref{eq: classifier loss} is equivalent to minimizing a Mone Carlo estimate of \cref{eq:loss-is-proof}.

\section{Experimental details}

\subsection{Checkerboard Experiment}
\paragraph{Architecture} We use a fully connected network with 2 residual blocks as shown in \cref{fig:toy-arch}. The hidden layer size in our experiment is 256 and timestep embeddings (output of the sinusiodal embedding layer) is 128. Our classifier has a similar architecture, the only difference is that the classifier has a different output dimension of one. Our baseline DM and classifier networks both have around 330k parameters.
\paragraph{Training the baseline DM} We train our models baseline models on a single GPU, (we use either of GeForce GTX 1080 Ti or GeForce GTX TITAN X) for 30,000 iterations. We use the Adam optimizer \citep{kingma2014adam} with a batch size of $3 \times 10^{-4}$ and full-batch training i.e., our batch size is 1000 which is the same as the training dataset size.
\paragraph{Training the classifiers} Each classifier is trained on a fully-synthetic dataset of 100k samples which consists of 50k positive and 50k negative samples. This dataset is generated with 100 diffusion steps. We train the classifier for 20k iterations with a batch size of 8192. We use Adam optimizer with a learning rate of $3 \times 10^{-3}$.
\paragraph*{Distillation} The distilled models have the same architecture and hyperparameters as the baseline DM model. They are trained for 250k iterations on the true dataset with a batch size of 1000. We use Adam optimizer with a learning rate of $3 \times 10^{-4}$.
\paragraph{Diffusion process} We use the EDM framework in this experiment with a preconditioning similar to the one proposed in \citet{karras2022elucidating}. In particular, the following precoditioning is applied to the the network in \cref{fig:toy-arch}, called $F_\theta(\rvx_t, t)$, to get $D_\theta(\rvx_t; t)$ which returns an estimate of $\rvx_0$.
\begin{equation}
    D_\theta(\rvx_t; t) = \frac{\sigma_{\text{data}}^2}{\sigma(t)^2 + \sigma_{\text{data}}^2} \rvx_t + \sigma(t)F_\theta\left( \frac{1}{\sqrt{\sigma(t)^2 + \sigma_{\text{data}}^2}} \rvx_t; \frac{1}{4} \ln(\sigma(t)) \right),
\end{equation}
where $\sigma_{\text{data}} = 1$. Since smaller noise levels are important in our application, we changed the training distribution of $t$ from the log-Normal used in \citet{karras2022elucidating} to a log-uniform with the support of $\sigma_{\text{max}} = 80$ and $\sigma_{\text{min}} = 2 \times 10^{-3}$. In total, we spent around 250 GPU-hours for this experiment.
\begin{figure}
    \centering
    \includegraphics[scale=1]{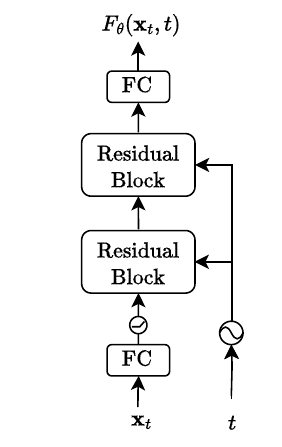}
    \hspace{10em}
    \includegraphics[scale=1]{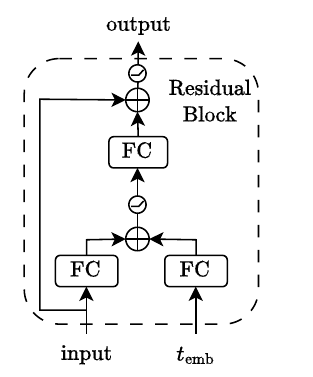}
    \caption{Architecture of the network in the checkerboard experiment. Left: the overall of the model. Right: detailed architecture of our ``Residual Block''. In this architecture, timestep is embedded using sinusoidal embedding and all nonlinearities are SiLU. The output of the network $F_\theta(\rvx_t, t)$ is then used in a preconditioning function to get an estimate of $\rvx_0$.}
    \label{fig:toy-arch}
\end{figure}
\paragraph{Sampling} We use the second-order Heun solver of \citet{karras2022elucidating} with 100 sampling steps and $S_{\text{churn}} = 10$. We modify the schedule of $\sigma$ to a log-linear schedule from $\sigma_{\text{max}}$ to $\sigma_{\text{min}}$.

\subsection{Traffic Scene Generation}\label{app:sec:details-ic}
\paragraph{Overview} This section provides additional details for the traffic scene generation task. The architectures for training the baseline DM model, classifiers and distillation models are majorly based on transformers introduced by \citet{vaswani2017attention} . In particular, the architecture backbone consists of an encoder, a stack of attention residual blocks, and a decoder. Each of them will be discussed in detail later.
The original data input shape is $[B, A, F]$ corresponding to $A$ vehicles and $F$ feature dimensions in a batch with $B$ many scenes.

In terms of parameters, the attention layers comprise the major portion of the entire architectures. This leads to the difference in decoder being relatively minor, and the resulting architectures all contain approximately $6.3$ million parameters. We use NVIDIA A100 GPUs for training and validating models, synthetic datasets generation with around $400$ GPU-hours in total. We train each model with a batch size of $64$ and Adam optimizer with a learning rate of $10^{-4}$.

\paragraph{Encoder and time embeddings}
To generate input features, we use sinusoidal positional embeddings to embed the diffusion time step and 2-layer MLP with activation function SiLU to embed the original data separately into $H = 196$ hidden feature dimensions. The sum of the two embeddings is the input that is fed into the attention-based architecture.

\paragraph{Self-attention and cross-attention layers}
The major implementation of multi-head ($k=4$) attention blocks is built on Transformer~\citep{vaswani2017attention}.
Applying self-attention across agents enables model to learn the multi-agent interactions, while applying map-conditional cross-attention between agents and map allows agents to interact with the road representations.
To prepare road image for model input, we use a convolutional neural network and a feed-forward network~\citep{carion2020end} to generate a lower-resolution map $m'\in\sR^{196\times 32\times 32}$ from the original image $m\in\sR^{3\times 256\times 256}$. Since the transformer architecture is permutation-invariant, we add a 2D positional encoding~\citep{parmar2018image, bello2019attention} based on $m'$ on the top of the map representation to preserve the spatial information of the image.

\paragraph{Relative Positional Encodings (RPEs)}During experiments, we find the collision rate is much higher than the offroad rate. In order to effectively lower the frequency of or completely avoid vehicle collision occurrence, we manage to capture the relative positions by performing relative positional encodings (RPEs) in self-attention residual blocks and enforce the vehicles being aware of the other vehicles in close proximity in each scene.
Following~\citet{shaw2018self, wu2021rethinking, harvey2022flexible}, we compute the distances of each pair of vehicles and summarise into a tensor of shape $[B, A, A]$, where $d^b_{ij}$ is the distance between vehicle $i$ and $j$ in the $b^{\text{th}}$ scene.
We choose to use sinusoidal embeddings (similar to how we embed diffusion time $t$) to parameterize $d^b_{ij}$ rather than logarithm function $f_{\text{RPE}}(d^b_{ij}) = \log(1 + d^b_{ij})$, as we need to adequately amplify the pairwise distances between vehicles when it is comparably small.
We perform this operation together with diffusion time embedding at each diffusion time step, and we regard their sum as the complete pairwise distance embeddings.
The resulting embedding tensor $\rvp$ is of the shape $[B, A, A, H]$, where $\rvp^b_{ij}$ is the encoding vector of length $H$ representing the pairwise distance of vehicle $i$ and $j$ in the $b^{\text{th}}$ scene.

In each scene, we have an input sequence, $\rvx = (\rvx_1, \cdots, \rvx_A)$, and each $\rvx_i$ is linearly transformed to query $\rvq_i = W^Q \rvx_i$, key $\rvk_i = W^K \rvx_i$ and value $\rvv_i = W^V \rvx_i$.
We also apply linear transformation onto RPEs to obtain query $\rvp_{ij}^Q = U^Q \rvp_{ij}$, key $\rvp_{ij}^K = U^K \rvp_{ij}$ and value $\rvp_{ij}^V = U^V \rvp_{ij}$.
Then the add-on output from the self-attention residual block is the aggregated outputs of the vanilla transformer and the relative-position-aware transformer:
\begin{align}
    \rvx_i^{\text{output}}        & = \rvx_i + \sum_{j=1}^A \alpha_{ij} (\rvv_j + \rvp_{ij}^V)                                                                                                           \\
    \text{where}\quad \alpha_{ij} & = \frac{\exp(e_{ij})}{\sum_{k=1}^A \exp(e_{ik})}\ \text{and}\  e_{ij} = \frac{\rvq_i^\top \rvk_j + \rvp_{ij}^{Q^\top}\rvk_j + \rvq_i^\top\rvp_{ij}^K}{\sqrt{d_\rvx}}
\end{align}

\paragraph{Decoder}
The settings for baseline, distillation models and classifiers are almost identical except the decoder for producing the final output. For baseline and distillation models, we apply 2-layer MLP and reconstruct the output of the shape $[B, A, H]$ from the final attention layer into $[B, A, D]$ through the decoder.
To ensure we output individual label for each vehicle with classifiers, we conduct the operations as follows.
The decoder takes the hidden representation of the shape $[B, A, H]$ and produces a tensor with feature dimension $F'=1$ with a 2-layer MLP, which is the predicted labels from the classifiers.

\subsection{Motion Diffusion}
For the task of text-conditional motion generation, we use the HumanML3D dataset \citep{guo2022generating}. This dataset contains 14,616 human motions annotated by 44,970 textual descriptions. It includes motions between 2 and 10 seconds in length and their total length amounts to 28.59 hours. Each motion is between $36$ and $196$ frames, with the majority of them comprising $196$ frames. Each frame is represented by a $263$-dimensional feature vector, resulting in a dimensionality of over $51,000$ for the largest motions.

We used the official implementation of MDM\footnote{\href{https://github.com/GuyTevet/motion-diffusion-model}{\url{https://github.com/GuyTevet/motion-diffusion-model}}} for our Motion Diffusion experiment. For the baseline DM, we used their officially released best pretrained checkpoint of text-to-motion task on HumanML3D dataset. We generate a synthetic dataset of around 250k positive and 250k negative examples from the baseline DM which is a DDPM-based model with 1000 diffusion steps. We then define our classifier architecture using their code base. Following our other experiments, our classifier architecture is the same as the baseline DM model. We train the classifier with a batch size of 128 and a learning rate of $10^{-4}$ for 300k iterations. Otherwise, we use the same hyperparameters as in \citet{tevet2023human}. All the training and data generation is done on A100 GPUs.

To compute the FID scores, in accordance with \citet{tevet2023human}, we generate one motion for each caption in the HumanML3D test set, resulting in a total of 4,626 generated motions.

The total compute used for this experiment (generating the datasets and training the classifiers) was around 600 GPU-hours.

\begin{figure}[tb]
    \centering
    \includegraphics[width=0.8\textwidth]{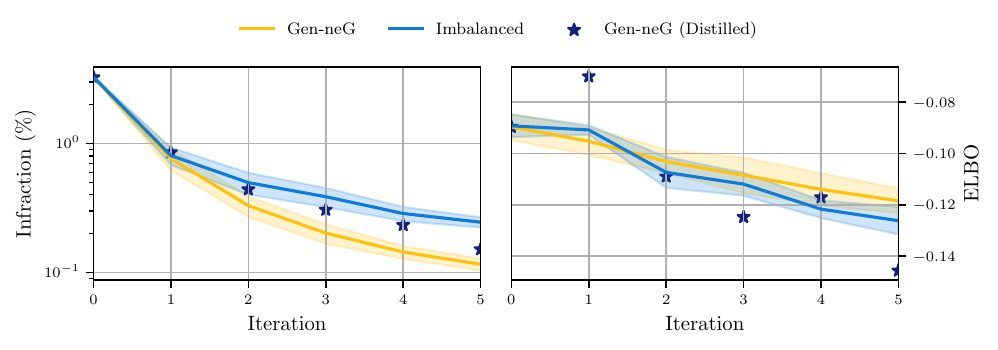}
    \caption{Infraction and ELBO estimations from different iterations of \method{}, distilled \method{} and an ablated version of it without label imbalance correction. \method{} achieves a lower infraction rate and a comparable ELBO.}
    \label{fig:2d-ablation-imbalance}
\end{figure}

\section{Ablation studies}

\subsection{Label imbalance}\label{app:ablation-imbalance}
As mentioned in \cref{sec:method:classifier}, generating synthetic datasets for training the classifiers without careful planning leads to significant label imbalance issues, which impede the training process. In this section we perform an ablation study to empirically demonstrate its detrimental impact in our experiments.

As a reminder, \method{} requires an equal number of positive and negative examples in the synthetic datasets, and it employs importance sampling to address any distribution shift introduced.
In the ablated experiment (referred to as "imbalanced" in the results), the ratio of positive to negative examples is model-dependent, being equal to the model's infraction rate.
It results in a gradual increase in the dominance of positive examples as the model's infraction rate decreases.

\paragraph*{Checkerboard experiment} We repeat our checkerboard experiment with and without \method{}'s imbalance correction. In each iteration, we create a synthetic dataset of 20,000 samples and train a binary classifier for 10,000 iterations. The other details are the same as the experiment in the main text. We show the performance of these models in \cref{fig:2d-ablation-imbalance}.
We observe that ``imbalaced'', the ablated version of \method{}, is consistently outperformed by \method{} in infraction rate. Furthermore, the performance gap between the two methods widens as the models improve.
However, it is worth noting that \method{} achieves a comparable ELBO to that of ``imbalanced'' despite these infraction rate differences.

\paragraph*{Traffic Scene Generation} We conduct a similar ablation as described above, where we train classifiers on a imbalanced datasets of the same size as our method. The results of this ablated experiment are presented in \cref{tab: 2nd experiment} in the main text.

\subsection{Synthetic dataset size}
We conducted an ablation on the number of samples required on the checkerboard and the traffic scene generation tasks.~\cref{fig:toy-ablation-n} and~\cref{tab:ic-ablation-n} show our results. We observe that the infraction rate constantly decreases irrespective of the dataset size. However, in order to avoid distribution shift we need a large enough dataset, as is evident from the ELBO plot. It is important to emphasize that the classifiers are trained on fully synthetic data generated by the model itself. Therefore, in principle we have access to an unbounded number of samples. As our results show, for the best performance, it is important to ensure the sample size is sufficiently large.

\section{Additional results}

\begin{figure}[!t]
    \centering
    \includegraphics[width=0.8\textwidth]{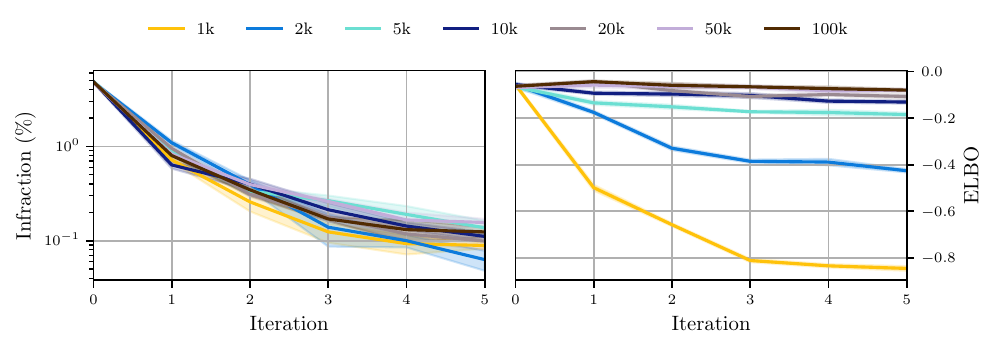}
    \caption{Ablation study for checkerboard experiment on synthetic dataset size}
    \label{fig:toy-ablation-n}
\end{figure}

\begin{table}[!t]
\centering
  \caption{Ablation study for the traffic scene generation experiment on synthetic dataset size}
  \label{tab: 2nd experiment ablation study}
  \begin{tabular}{lllll}
    \toprule
    Classifier dataset size     & Collision ($\%$) $\downarrow$  & Offroad ($\%$) $\downarrow$ & Infraction ($\%$) $\downarrow$ & r-ELBO ($\times 10^{-2}$) $\uparrow$  \\
    \midrule
    baseline DM  & $28.3\pm 0.7$  & $1.3\pm 0.1$ & $29.3\pm 0.6$ &  $-27.5\pm0.01$   \\
        \midrule
    $8,000$ & $23.8 \pm 0.4$ & $0.9\pm 0.2$ & $24.5 \pm 0.5$ & $-28.0\pm 0.01$ \\
    $40,000$ & $19.7\pm 0.8$ & $0.8\pm 0.2$ & $20.3\pm 0.9$ & $-27.8\pm0.01$\\
    $80,000$ & $17.6\pm 0.7$ & $0.7\pm 0.2$ & $18.2\pm 0.6$ & $-27.7\pm 0.01$\\
    $400,000$ & $16.6 \pm 0.7$ & $0.8\pm 0.2$ & $17.5\pm 0.7$ & $-27.7\pm 0.01$ \\
    $800,000$ & $16.4\pm 0.5$ &  $0.9\pm 0.1$ & $17.2\pm 0.4$ & $-27.7\pm 0.01$ \\
    \bottomrule
  \end{tabular}
  \label{tab:ic-ablation-n}
\end{table}

\subsection{Computational cost and sampling latency} \label{sec:app:computational-cost}
Here we discuss the latency of different models considered in this paper. We first report the wall-clock time of generating samples from models. However, in order to deploy any of these models, one should ensure a generated sample is valid. Therefore, all the models should be used together with rejection sampling. We discuss the latency in conjunction with rejection sampling in the second part of this section.

\paragraph{Latency} We compute the average wall-clock time of the baseline models and iterations of \method{} for different experiments and present the results in~\cref{tab:comp-time}.
We observe that additional classifiers linearly increase the running time (for conciseness, we have not included further iterations of \method{} on checkerboard experiments in \cref{tab:comp-time} as its runtime simply continues growing linearly). Moreover, in the Traffic Scenes and checkerboard experiments, the overhead of each classifier is larger than the runtime of the baseline model alone (almost 1.5 times). This is because the architecture of our classifier is the same as the baseline model and for each forward pass of a classifier-guided model, one forward and one backward pass through the classifier is required. However, in the Motion diffusion experiment, since the classifier is not text-conditional, this overhead is relatively smaller. We also report the total time to generate one batch of samples for all the models in \cref{tab: 2nd experiment}. This table also includes the number of parameters for each model as a proxy for memory consumption. It is important to note that our distilled models have the exact same latency and memory consumption as the baseline diffusion model, since we use the same architecture for the distilled model.

\paragraph{Rejection sampling} To ensure that a sample from a model is valid, deploying any of these models requires using them with rejection sampling.
Assume a model with infraction rate of $\epsilon$ and sampling time of $t$ seconds is given. A rejection sampling loop with this model takes $\frac{1}{1 - \epsilon} t$ seconds. However, in a different, more realistic setting, we require an accepted sample in ``one iteration'' of running the model. Since all models have some infraction rate, we instead generate a batch of samples and require at least one non-infracting sample with high probability. As stated in \cref{sec:intro}, generating at least one non-infracting sample with $(1 - \delta)$ probability requires $\frac{\log \delta}{\log \epsilon}$ parallel samples. Therefore, any improvement to $\epsilon$ leads to improved computational complexity. In particular,  we get the following reductions in the required number of samples:

\begin{table}[!t]
\centering
    \caption{Wall-clock time of one denoising step}
    \begin{tabular}{llll}
        \toprule
        & \multicolumn{3}{c}{Time per diffusion step (s) ($\times 10^{-3})$}\\
        \cmidrule{2-4}
        Method & Checkerboard & Traffic Scenes & Motion Diffusion \\
        \midrule
        Baseline DM / distilled \method{} & $0.2 \pm 1 \times 10^{-2}$ & $32 \pm 9\times 10^{-3}$ & $83 \pm 2 \times 10^{-1}$ \\
        \method{} (iteration 1) & $0.5 \pm 2 \times 10^{-2}$ & $86 \pm 2\times 10^{-2} $ & $171 \pm 8 \times 10^{-1}$ \\
        \method{} (iteration 2) & $0.8 \pm 2 \times 10^{-2}$ & $138 \pm 3\times 10^{-2}$ & - \\
        \bottomrule
    \end{tabular}
    \label{tab:comp-time}
\end{table}

\begin{table}[!t]
\centering
    \caption{Sampling time of the traffic scenes experiment (1000 steps)}
    \begin{tabular}{lll}
    \toprule
    Method & Parameters (M)  & Latency (s) \\
    \midrule
    Baseline DM & $6.3$ & $38.58 \pm 0.20$ \\
    \method{} (iteration 1) & $12.6$ & $86.71 \pm 0.19$ \\
    \method{} (iteration 2) & $18.9$ & $138.73 \pm 0.42$ \\
    \method{} (distilled) & $6.3$ & $38.43 \pm 0.37$ \\
    Time-Independent classifier & $12.6$ & $126.40 \pm 0.20$ \\
    Imbalanced classifier & $12.6$ & $86.88 \pm 0.20$ \\
    w/o IS classifier & $12.6$ & $86.81 \pm 0.21$ \\
    \bottomrule
    \end{tabular}
    \label{tab:latency-ic}
\end{table}

\begin{itemize}
    \item 47\% on the checkerboard experiment (baseline vs. distillation of 5 iterations of Gen-neG),
    \item 57\% on the traffic scenes experiment (baselines vs. distillation of stacked Gen-neG),
    \item 9\% on the motion diffusion experiment.
\end{itemize}

We conjecture the relatively smaller improvement in the motion diffusion experiment is because the baseline DM predicts $\rvx_0$. A follow up to MDM, argues that $\rvx_0$-prediction models are hard to guide (see appendix A of \citet{zhong2022guided}).

\subsection{Faster synthetic dataset generation}
With the iterative stacking of the classifiers in \method{}, as the model becomes better in avoiding invalid samples, it becomes harder to generate a balanced synthetic dataset. To further reduce the computational cost for cases where the infraction rate is very small, one can employ Sequential Importance Sampling (SIS) using a proposal distribution with higher infraction rate. Concretely, assume we use an SDE solver with pre-defined timesteps $t_0=0 < t_1 < \ldots < t_T=1$ to generate the samples. Therefore the loss function in \cref{eq:gen-neg-true-objective} becomes 
\begin{equation}
    \gL_\phi^{\text{cls}} = \mathbb{E}_{p_\theta(\mathbf{x}_0)}\Big[l(\mathbf{x}_0)\Big] = \mathbb{E}_{p_\theta(\mathbf{x}_{t_0}, \ldots, \mathbf{x}_{t_T})}\Big[l(\mathbf{x}_0)\Big] = \mathbb{E}_{\pi(\mathbf{x}_{t_0}, \ldots, \mathbf{x}_{t_T})}\Big[\frac{p_\theta(\mathbf{x}_{t_0}, \ldots, \mathbf{x}_{t_T})}{\pi(\mathbf{x}_{t_0}, \ldots, \mathbf{x}_{t_T})} l(\mathbf{x}_0)\Big],
\end{equation}
where $l(\rvx_0) = \meanp{t, q(\rvx_t | \rvx_0)}{ \oracle{}(\rvx_0) \log \cls_\phi(\rvx_t; t) + (1 - \oracle{}(\rvx_0)) \log (1 - \cls_\phi(\rvx_t; t)) }$ and $\pi$ is a proposal distribution defined as another diffusion model for example, an earlier iteration of Gen-neG. Therefore, $\frac{p(\mathbf{x}_{t_0}, \ldots, \mathbf{x}_{t_T})}{\pi(\mathbf{x}_{t_0}, \ldots, \mathbf{x}_{t_T})} = \prod_{i=1}^T\frac{p(\mathbf{x}_{t-1} | \mathbf{x}_t)}{\pi(\mathbf{x}_{t-1} | \mathbf{x}_t)}$ in which both the numerator and denominator are Gaussian distributions with the same variances but different means.

\begin{figure}
    \centering
    \includegraphics[width=\linewidth]{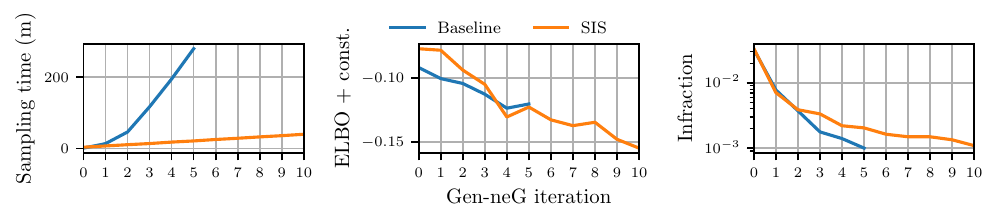}
    \caption{The results of the SIS approach for faster synthetic dataset generation for \method{}. SIS makes the sampling time grow linearly with iterations of Gen-neG while maintaining a comparable performance to the standard synthetic dataset sampling approach in \method{}.}
    \label{fig:app:faster-sampling}
\end{figure}

\cref{fig:app:faster-sampling} shows the results of this approach on the checkerboard experiment. It shows that SIS dataset sampling time grows linearly and independent of the infraction rate while maintaining a comparable performance. This linear growth is due to the linear growth in complexity of (non-distilled) Gen-neG models.

\begin{table}[!t]
\centering
  \caption{Comparison of a diffusion model and normalizing flow model trained on the INTERACTION dataset. The normalizing flow results are taken from \citet{zwartsenberg2023conditional}.}
  \label{tab:app:ic-interaction}
  \begin{tabular}{llll}
    \toprule
    Method    & Collision ($\%$) $\downarrow$  & Offroad ($\%$) $\downarrow$ & Infraction ($\%$) $\downarrow$  \\
    \midrule
    Diffusion model  & $7.34\pm 0.16$  & $6.94\pm 0.06$ & $13.62\pm 0.13$   \\
    Normalizing flow~\citep{zwartsenberg2023conditional} & $9.00\pm 1.00$& $12.00\pm 1.00$ & $20.00\pm 1.00$ \\
    \bottomrule
  \end{tabular}
\end{table}

\subsection{INTERACTION dataset}\label{sec:app:ic-interaction}
Here we report our results of training a baseline DM model on the INTERACTION dataset~\citep{zhan2019interaction}. \cref{tab:app:ic-interaction} shows the superior performance of our diffusion model compared to the normalizing flow results from \citet{zwartsenberg2023conditional}. The lower infraction rates compared to \cref{tab: 2nd experiment} suggests that the INTERACTION dataset is a simpler dataset compared to the one we used in \cref{sec:IC-experiment}. One can further improve upon the diffusion model in \cref{tab:app:ic-interaction} by using it as a baseline DM in \method{}.

\subsection{More visualization results for the traffic scene generation experiment}\label{app:ic-more-results}
In \cref{fig:ic-more-results} we report more visualization results from our traffic scene generation experiment. This figure follows from and adds more details to \cref{fig:banner}.

\subsection{Overfitting in the checkerboard experiment}\label{app:overfitting}
Here we present our results regarding the overfitting of the baseline DM in the checkerboard experiment. We run an experiment with 200,000 training iterations, much larger than the 30,000 iterations in the reported results. As we can see in \cref{fig:toy-overfit}, the infraction rate keeps decreasing. However, the model starts overfitting after around 30,000 iterations, as measured by the ELBO on a held-out set. This suggests that the architecture is expressive enough to model sharp jumps in the learned density. However, simply training it on a small dataset without incorporating any prior on ``where to allocate its capacity'' fails because the model does not receive any signal on where the actual ``sharp jump'' is. \method{}, on the other hand, provides this kind of signal through the oracle-assisted guidance.
\begin{figure}[t]
    \subfigure[]{
    \includegraphics[width=0.18\textwidth]{figs/plot_kde_6.pdf}
    \label{fig:ic-more-results-a}
    }
    \hfill
    \subfigure[]{
    \includegraphics[width=0.18\textwidth]{figs/plot_kde_7.pdf}
    \label{fig:ic-more-results-b}
    }
    \hfill
    \subfigure[]{
    \includegraphics[width=0.18\textwidth]{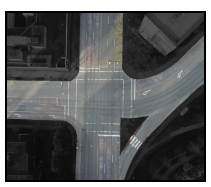}
    \label{fig:ic-more-results-c}
    }
    \hfill
    \subfigure[]{
    \includegraphics[width=0.18\textwidth]{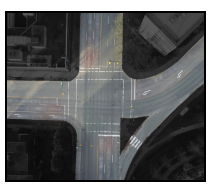}
    \label{fig:ic-more-results-d}
    }
    \hfill
    \subfigure[]{
    \includegraphics[width=0.18\textwidth]{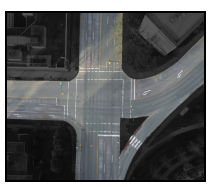}
    \label{fig:ic-more-results-e}
    }
    \caption{Complete visualization comparisons for infraction in traffic scene generation experiments. Subplots show infraction per unit area under different models. \subref{fig:ic-more-results-a} baseline DM; \subref{fig:ic-more-results-b} first iteration of \method{}; \subref{fig:ic-more-results-c} second iteration of \method{}. \subref{fig:ic-more-results-d} and \subref{fig:ic-more-results-e} are the distillation models corresponding to \subref{fig:ic-more-results-b} and \subref{fig:ic-more-results-c} respectively. A clear reduction in terms of infractions per unit area can be observed from left to right.}
    \label{fig:ic-more-results}
\end{figure}
\begin{figure}[t]
    \centering
    \includegraphics[width=\textwidth]{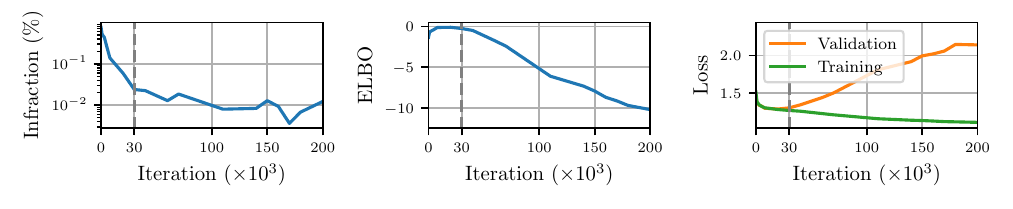}
    \caption{Overfitting results in the checkerboard experiment. The first two plots on the left respectively show the infraction rate and ELBO on a held-out validation set. We observe that training the baseline DM for longer can achieve much lower infraction rates. However, it quickly starts to overfit, leading to poor ELBO estimates on the held-out validation set. The last plot shows the training and validation loss of the model. These plots confirm that the checkpoint we used for baseline DM in \method{} at 30,000 iterations does not overfit.}
    \label{fig:toy-overfit}
\end{figure}
\end{document}